%% file: omp_newsc.tex
\documentclass[english,10pt,journal]{IEEEtran}
\usepackage{color,graphicx,amsmath,amssymb,amsthm,epsfig,mathrsfs,cite,bm,enumerate}
\usepackage{array}
\usepackage{verbatim}
\usepackage{mathtools}

\DeclarePairedDelimiter\floor{\lfloor}{\rfloor}
\usepackage{color,soul}
\usepackage{multirow}
\makeatletter
\newcommand*{\rom}[1]{\expandafter\@slowromancap\romannumeral #1@}
\makeatother

\include{notation}
\theoremstyle{remark} \newtheorem{remark}{Remark}
\usepackage{amssymb}
\usepackage{float}
\usepackage{tikz}
\usetikzlibrary{trees}
\usepackage{algorithm}
\usepackage{algpseudocode}
\usepackage{lipsum}
\usepackage[lofdepth,lotdepth]{subfig}
\usepackage{enumerate}

\title{Tuning Free Orthogonal Matching Pursuit} 

\author{Sreejith Kallummil,  \hspace{0cm} Sheetal Kalyani  \\
 Department of Electrical Engineering \\  Indian Institute of Technology Madras\\
  Chennai, India 600036 \\
  \{ee12d032,skalyani\}@ee.iitm.ac.in
  }

\begin{document}
\maketitle
\begin{abstract}
Orthogonal matching pursuit (OMP) is a widely used compressive sensing (CS) algorithm for recovering sparse signals in noisy linear regression models.  The performance of OMP depends on its stopping criteria (SC). SC for OMP discussed in literature  typically assumes  knowledge of either the sparsity of the  signal  to be estimated $k_0$ or noise variance $\sigma^2$, both of which are unavailable in many practical applications. In this article we develop a modified version of OMP called tuning free OMP or TF-OMP which does not require a SC. TF-OMP is proved to accomplish successful sparse recovery  under the usual assumptions on restricted isometry constants (RIC) and mutual coherence of design matrix.  TF-OMP is numerically shown to  deliver a  highly competitive performance in comparison with OMP having \textit{a priori} knowledge of $k_0$ or $\sigma^2$. Greedy algorithm for robust de-noising (GARD) is an OMP like algorithm proposed for efficient estimation in
classical overdetermined linear regression models corrupted by sparse outliers. However, GARD requires the knowledge of inlier noise variance which is difficult to estimate.
We also produce a tuning free algorithm (TF-GARD) for efficient estimation  in the presence of sparse outliers by extending the operating principle of TF-OMP to GARD.   TF-GARD is numerically shown to achieve a performance comparable to that of the existing implementation of GARD.    
\end{abstract}

\section{Introduction}

Consider the linear regression model ${\bf y}={\bf X}\boldsymbol{\beta}+{\bf w}$, where ${\bf X} \in \mathbb{R}^{n \times p}$ is a known design matrix, ${\bf w}$ is the noise vector and ${\bf y}$ is the observation vector. The design matrix is  rank deficient in the sense that $rank({\bf X})< p$.   Further, the columns of ${\bf X}$ are normalised to have unit Euclidean $(l_2)$ norm. The vector $\boldsymbol{\beta}\in \mathbb{R}^{n \times p}$ is sparse, i.e., the support of $\boldsymbol{\beta}$ given by $\mathcal{I}=supp(\boldsymbol{\beta})=\{k:\boldsymbol{\beta}_k\neq 0\}$ has cardinality $k_0=|\mathcal{I}|\ll p$. The noise vector ${\bf w}$ is assumed to have Gaussian distribution with mean ${\bf 0}_n$ and covariance $\sigma^2{\bf I}_n$, i.e., ${\bf w} \sim \mathcal{N}({\bf 0}_n,\sigma^2{\bf I}_n)$.   The signal to noise ratio in this regression model is defined  as 
\begin{equation}
SNR=\dfrac{\|{\bf X\boldsymbol{\beta}}\|_2^2}{\mathbb{E}(\|{\bf w}\|_2^2)}=\dfrac{\|{\bf X\boldsymbol{\beta}}\|_2^2}{n\sigma^2}.
\end{equation}
 Throughout this paper, $\mathbb{E}()$ represents the expectation operator and $\|{\bf x}\|_q=\left( \sum\limits_{k=1}^p|\boldsymbol{\beta}_k|^q\right)^{\frac{1}{q}} $ represents the $l_q$ norm of ${\bf x}$.  In this article we consider the following  two  problems  in the context of recovering sparse vectors in underdetermined linear regression models which are of larger interest. 

P1). Estimate ${\boldsymbol{\beta}}$ with the objective of minimizing the mean squared error (MSE) $MSE(\hat{{\boldsymbol{\beta}}})=\mathbb{E}(\|{\boldsymbol{\beta}}-\hat{{\boldsymbol{\beta}}}\|_2^2)$.

P2). Estimate the support of ${\boldsymbol{\beta}}$ with the objective of minimizing the probability of support recovery error $PE(\hat{\boldsymbol{\beta}})=\mathbb{P}(\hat{\mathcal{I}}\neq \mathcal{I})$, where $\hat{\mathcal{I}}=supp(\hat{\boldsymbol{\beta}})$. \\

These problems  are common in  signal processing applications like sparse channel estimation\cite{channel}, direction of arrival estimation\cite{single_snap}, multi user detection\cite{multiuserCS} etc. Typical machine learning applications include sparse subspace clustering\cite{subspace}, sparse representation classification\cite{facerecognition} etc. In signal processing community these problems are  discussed under the  compressive sensing (CS) paradigm\cite{eldar2012compressed}.  A number of algorithms like least absolute shrinkage and selection operator (LASSO)\cite{tropp2006just,candes2009near}, Dantzig selector (DS)\cite{candes2007dantzig}, subspace pursuit (SP)\cite{subspacepursuit}, compressive sampling matching pursuit (CoSaMP)\cite{cosamp}, sparse Bayesian learning (SBL)\cite{wipf2004sparse}, orthogonal matching pursuit (OMP)\cite{tropp2004greed,cai2011orthogonal,OMP_wang,tropp2007signal,omp_decay,extra,brownian} etc. are proposed to solve the above mentioned problems. However, for optimal performance of these algorithms, a number of tuning parameters (also called hyper parameters) need to be fixed. For example, the value of $\lambda$ in LASSO estimate 
\begin{equation}\label{lasso}
\hat{\boldsymbol{\beta}}=\underset{{\bf b}}{\arg\min}\dfrac{1}{2}\|{\bf y}-{\bf Xb}\|_2^2+\lambda\|{\bf b}\|_1
\end{equation}   
has to tuned appropriately. Indeed, when the noise is Gaussian a value $\lambda=O(\sigma\sqrt{2\log(p)})$ is known to be optimal in terms of MSE performance\cite{candes2009near}. Likewise, a value of $\lambda\propto \sigma^{1-\alpha}$ with $0<\alpha<1$  is known to deliver $PE \rightarrow 0$ as $\sigma^2 \rightarrow 0$ under some regularity conditions\cite{cs_tsp}. Likewise, for the optimal performance of DS, one need to have knowledge of $\sigma^2$\cite{candes2007dantzig}.  However, unlike the case of overdetermined linear regression models where one can readily estimate $\sigma^2$ using the maximum likelihood (ML) estimator, estimating $\sigma^2$ in underdetermined linear regression models is extremely difficult\cite{giraud2012}. This means that the optimal performance using LASSO and DS in many practical applications involving Gaussian noise\footnote{Apart from the Gaussian noise model considered in this paper, two other noise models popular in literature are $l_2$ bounded noise $(\|{\bf w}\|_2<\epsilon_2)$ and $l_{\infty}$ bounded noise $(\|{\bf w}\|_{\infty}<\epsilon_{\infty})$. The optimal performance of aforementioned algorithms in these models requires the knowledge of parameters $\epsilon_2$ and $\epsilon_{\infty}$ which are also difficult to estimate.} is not possible. Even if the noise variance is known, an amount of subjectivity is involved in fixing the tuning parameters. SBL on the other hand involves a non convex optimization problem  which is solved using the expectation maximization (EM) algorithm and hence the solution depend on the initialization values of EM algorithm.   Likewise, algorithms like CoSaMP, SP etc. requires \textit{ a priori} knowledge of sparsity level $k_0$ which is rarely available. OMP, which is the focus of this article, requires either the knowledge of $k_0$ or the knowledge of $\sigma^2$ for optimal performance.  Hence, in many practical applications, the statistician is forced to choose ad hoc  values of tuning parameters for which no performance guarantees are available. A popular alternative is  based on techniques like cross validation which can deliver reasonably good performance at the expense of significantly high computational complexity\cite{crossvaldation,cross_val,giraud2012}. Further, cross validation is also known to be ineffective for  support recovery problems\cite{cross_val}. 
\subsection{Tuning parameter free sparse recovery.}
The literature on tuning parameter free sparse recovery procedure is new in comparison with the literature on sparse recovery algorithms like OMP, LASSO, DS etc. A seminal contribution in this field is the square root LASSO\cite{sqlasso} algorithm which estimate $\boldsymbol{\beta}$ by
\begin{equation}
\hat{\boldsymbol{\beta}}=\underset{{\bf b}}{\arg\min}\|{\bf y}-{\bf Xb}\|_2+\lambda\|{\bf b}\|_1
\end{equation}
For optimal MSE performance $\lambda$ can be set independent of $\sigma^2$ thereby overcoming a major drawback of LASSO. However, the choice of $\lambda$ is still subjective with little guidelines.  The high SNR behaviour of PE for square root LASSO is not reported in the literature. Another interesting development in this area is the development of sparse iterative covariance-based estimation,  popularly called as SPICE\cite{spice}. SPICE is a convex optimization based algorithm that is completely devoid of any hyper parameters. The relationship between SPICE and techniques like LAD-LASSO, square root LASSO and LASSO are  derived in \cite{spice_connection,spicenote}. Another tuning parameter free algorithm called LIKES which is closely related to SPICE is proposed in \cite{spice_like}. Another interesting contribution in this area is the derivation of analytical properties of the non negative least squares (NNLS) estimator 
\begin{equation}
\hat{\boldsymbol{\beta}}= \underset{{\bf b}\geq {\bf 0}_p}{\arg\min}\|{\bf y}-{\bf Xb}\|_2^2
\end{equation}
in \cite{nnls} which points to the superior performance of NNLS in terms of MSE. However, the NNLS  estimate is applicable only to the cases where the sign pattern of $\boldsymbol{\beta}$ is known \textit{a priori}.  Existing literature on tuning free sparse recovery has many disadvantages. In particular, all these techniques are computationally complex in comparison with simple algorithms like OMP, CoSaMP etc. Notwithstanding the connections established between algorithms like SPICE and LASSO, the performance guarantees of SPICE   are not well established.    
\subsection{Robust regression in the presence of sparse outliers.}
In addition to the  recovery of sparse signals in underdetermined linear regression models (which is the main focus of this article), we also consider a regression model widely popular in robust statistics called sparse outlier model. Here we consider the regression model
\begin{equation}
{\bf y}={\bf X}\boldsymbol{\beta}+{\bf w}+{\bf g},
\end{equation} 
where ${\bf X} \in \mathbb{R}^{n\times p}$ is a full rank design matrix with $n>p$ or $n\gg p$, regression vector $\boldsymbol{\beta}$ may or may not be sparse and inlier noise ${\bf w} \sim \mathcal{N}({\bf 0}_n,\sigma^2{\bf I}_n)$. The outlier noise ${\bf g}$ represents the large   errors in the regression equation that are not modelled by the inlier noise distribution.  In many cases of practical interest, ${\bf g}$ is modelled as sparse, i.e., $n_{out}=|supp({\bf g})|\ll n$. However, the non zero entries in ${\bf g}$ can take large values, i.e., $\|{\bf g}\|_{\infty}$ can be potentially high.   Algorithms from robust statistics like  Hubers' M-est\cite{maronna2006wiley}  were used to solve this problem. Recently, a number of algorithms that utilizes the sparse nature of ${\bf g}$  like the convex optimization based \cite{error_correction_candes,koushik}, SBL based \cite{bdrao_robust}, OMP based greedy algorithm for robust de-noising (GARD)\cite{gard} etc. are shown to  estimate $\boldsymbol{\beta}$ more efficiently than the robust statistics based techniques.  Just like the case of sparse regression, algorithms proposed for robust estimation in the presence of sparse outliers also require tuning parameters that are subjective and dependent on inlier noise variance $\sigma^2$ (which is difficult to estimate). 

\subsection{ Contribution of this article.}
This article makes the following contributions to the CS literature.  We propose a novel way of using the popular OMP called tuning free OMP (TF-OMP) which does not require \textit{a priori} knowledge of sparsity level $k_0$ or noise variance $\sigma^2$ and is completely devoid of any tuning parameters. We analytically establish that the TF-OMP can recover  the true support $\mathcal{I}$ in $l_2$ bounded noise ($\|{\bf w}\|_2\leq \epsilon_2$) if  the matrix ${\bf X}$ satisfy either exact recovery condition (ERC)\cite{tropp2004greed}, mutual incoherence condition (MIC) \cite{cai2011orthogonal} or the restricted isometry condition in \cite{omp_rip_noise} and the minimum non zero value $\boldsymbol{\beta}_{min}=\underset{k \in \mathcal{I}}{\min}|\boldsymbol{\beta}_k|$ is large enough.  It is important to note that the conditions imposed on design matrix ${\bf X}$ for successful support recovery using TF-OMP is no more stringent than   the  results available \cite{cai2011orthogonal,OMP_wang,omp_rip_noise} in literature for OMP with \textit{a priori} knowledge of $k_0$ or noise variance $\sigma^2$.  Under the same set of conditions on matrix ${\bf X}$, TF-OMP is shown  to achieve high SNR consistency\cite{tsp,spl,cs_tsp} in Gaussian noise, i.e.,  $PE \rightarrow 0$ as $\sigma^2 \rightarrow 0$. This is the first time a tuning free CS algorithm is shown to achieve high SNR consistency.  As mentioned before, GARD for estimation in the presence of sparse outliers is closely related to OMP. We extend the operating principle behind TF-OMP to GARD and develop a modified version of GARD called TF-GARD which is devoid of tuning parameters and does not require the knowledge of inlier noise variance $\sigma^2$. Both  proposed algorithms, \textit{viz}. TF-OMP and TF-GARD are numerically shown to achieve highly competitive performance in comparison with a broad class of existing algorithms over a number of experiments.

\subsection{ Notations used. }
 $col({\bf X})$ the column space of ${\bf X}$. ${\bf X}^T$ is the transpose and ${\bf X}^{\dagger}=({\bf X}^T{\bf X})^{-1}{\bf X}^T$ is the  Moore-Penrose pseudo inverse of ${\bf X}$ (if ${\bf X}$ has full column rank). ${\bf P}_{\bf X}={\bf X}{\bf X}^{\dagger}$ is the projection matrix onto $col({\bf X})$.  ${\bf X}_{\mathcal{J}}$ denotes the sub-matrix of ${\bf X}$ formed using  the columns indexed by $\mathcal{J}$. ${\bf X}_{i,j}$ is the $[i,j]^{th}$ entry of ${\bf X}$. If ${\bf X}$ is clear from the context, we use the shorthand ${\bf P}_{\mathcal{J}}$ for ${\bf P}_{{\bf X}_{\mathcal{J}}}$. Both ${\bf a}_{\mathcal{J}}$ and ${\bf a}(\mathcal{J})$ denotes the  entries of ${\bf a}$ indexed by $\mathcal{J}$.   $\chi^2_j$ is a central chi square distribution with $j$ degrees of freedom (d.o.f). $\mathcal{C}\mathcal{N}({\bf u},{\bf C})$ is a complex Gaussian R.V with mean ${\bf u}$ and covariance matrix ${\bf C}$.  ${\bf a}\sim{\bf b}$ implies that ${\bf a}$ and ${\bf b}$ are identically distributed.    $\|{\bf A}\|_{m,l}=\underset{\|{\bf x}\|_m=1}{\max}{\|{\bf Ax}\|_l} $  is the $(m,l)^{th}$ matrix norm. $[p]$ denotes the set $\{1,\dotsc,p\}$. $\floor{x}$ denotes the floor function. $\phi$ represents the null set. For any two index sets $\mathcal{J}_1$ and $\mathcal{J}_2$, the set difference  $\mathcal{J}_1/\mathcal{J}_2=\{j:j \in \mathcal{J}_1\& j\notin  \mathcal{J}_2\}$. For any index set $\mathcal{J}\subseteq [p]$, $\mathcal{J}^C$ denotes the  complement of $\mathcal{J}$ with respect to $[p]$. $f(n)=O(g(n))$ iff $\underset{n \rightarrow \infty}{\lim}\frac{f(n)}{g(n)}<\infty$. 
\subsection{ Organization of this article:-} Section \rom{2} discuss  existing literature on OMP. Section \rom{3} present TF-OMP. Section \rom{4} presents the performance guarantees for TF-OMP. Section \rom{5} discuss TF-GARD algorithm. Section \rom{6} presents the numerical simulations. 

\section{OMP: Prior art}
The proposed tuning parameter free sparse recovery algorithm is based on OMP. OMP is a greedy procedure to perform sparsity constrained least square minimization. OMP starts with a null model and add columns  to current support  that is most correlated with the current residual. An algorithmic description of OMP is given in TABLE \ref{tab:omp}. The performance of OMP is determined by the properties of the measurement matrix ${\bf X}$, ambient SNR, sparsity of $\boldsymbol{\beta}$ ($k_0$) and  stopping condition (SC). We first describe the properties of ${\bf X}$ that are conducive for sparse recovery using OMP. 
\begin{table}
\begin{tabular}{|l|}
\hline
  {\bf Step 1:-} Initialize the residual ${\bf r}^{(0)}={\bf y}$. $\hat{\boldsymbol{\beta}}={\bf 0}_p$,\\ 	\ \ \ \ \ \ \ \ \ \ \   Support estimate ${\mathcal{J}^0}=\phi$, Iteration counter $k=1$; \\
  {\bf Step 2:-} Find the column  most correlated with the current \\
 \ \ \ \ \ \ \ \ \ \ \ \  residual ${\bf r}^{(k-1)}$, i.e., ${t_k}=\underset{t \in \{1,\dotsc,p\}}{\arg\max}|{\bf X}_t^T{\bf r}^{(k-1)}|.$ \\
 {\bf Step 3:-} Update support estimate: ${\mathcal{J}^k}={\mathcal{J}^{k-1}}\cup {t_k}$. \\
  {\bf Step 4:-} Estimate $\boldsymbol{\beta}$ using current support: $\hat{\boldsymbol{\beta}}(\mathcal{J}^k)={\bf X}_{\mathcal{J}^k}^{\dagger}{\bf y}$. \\
  {\bf Step 5:-} Update residual: ${\bf r}^{(k)}={\bf y}-{\bf X}\hat{\boldsymbol{\beta}}=({\bf I}_n-{\bf P}_{\mathcal{J}^k}){\bf y}$. \\
  {\bf Step 6:-} Increment $k$. $k \leftarrow k+1$. \\
  {\bf Step 7:-} Repeat Steps 2-6, until stopping condition (SC)  is  met. \\
  {\bf Output:-} $\hat{\mathcal{I}}=\mathcal{J}^k$ and $\hat{\boldsymbol{\beta}}$. \\
 \hline
\end{tabular}
\caption{ Orthogonal Matching Pursuit}
\label{tab:omp}
\end{table}

\subsection{ Qualifiers for design matrix ${\bf X}$.}
 When $n<p$, the  linear equation ${\bf y}={\bf X}{\boldsymbol{\beta}}$ has infinitely many possible solutions. Hence the support recovery problem is ill-posed even in the noiseless case. To uniquely recover the $k_0$-sparse vector $\boldsymbol{\beta}$, the measurement matrix ${\bf X}$ has to satisfy certain well known regularity conditions.  A plethora of sufficient conditions including restricted isometry property (RIP)\cite{eldar2012compressed,OMP_wang}, mutual incoherence condition (MIC)\cite{tropp2006just,cai2011orthogonal}, exact recovery condition (ERC)\cite{tropp2004greed,tropp2006just} etc. are discussed in the literature.  We first describe the ERC.  

{\bf Definition 1:-} A matrix ${\bf X}$  and a vector $\boldsymbol{\beta}$ with support $\mathcal{I}$ satisfy ERC if the exact recovery coefficient $erc({\bf X},\mathcal{I})=\underset{j \notin \mathcal{I}}{\max}\|{\bf X}_{\mathcal{I}}^{\dagger}{\bf X}_j\|_1$ satisfies $erc({\bf X},\mathcal{I})<1$.

It is known that ERC is a sufficient and worst case necessary condition for accurately recovering  $\mathcal{I}$ from ${\bf y}={\bf X}\boldsymbol{\beta}$ using OMP\cite{tropp2004greed}. The same condition with appropriate scaling of $\boldsymbol{\beta}_{min}$ is sufficient for  recovery in regression models with noise\cite{cai2011orthogonal}. Since  ERC  involves the unknown support $\mathcal{I}$, it is impossible to check ERC in practice. Another important metric used for qualifying ${\bf X}$ is the restricted isometry constant (RIC). RIC of order $k$ denoted  by $\delta_k$ is defined as the smallest value of $0\leq \delta\leq 1$ that satisfies
\begin{equation}
(1-\delta)\|{\bf b}\|_2^2\leq \|{\bf Xb}\|_2^2\leq (1+\delta)\|{\bf b}\|_2^2 
\end{equation}
for all $k$-sparse ${\bf b}\in \mathbb{R}^p$. OMP can recover a $k_0$ sparse signal $\boldsymbol{\beta}$ in the first $k_0$ iterations if $\delta_{k_0+1}<\frac{1}{\sqrt{k_0}+1}$\cite{OMP_wang,wang2012recovery,omp_rip_noise}. In the absence of noise, OMP can recover a $k_0$ sparse $\boldsymbol{\beta}$ in $[3.8k_0]$ iterations if $\delta_{[3.8k_0]}<2 \times 10^{-5}$ \cite{extra}. Likewise, it is possible to recover $\boldsymbol{\beta}$ in $30k_0$ iterations if $\delta_{31k_0}<\dfrac{1}{2}$\cite{extra}.   It is well known that the computation of RIC is NP-hard. Hence, mutual coherence, a quantity that can be estimated easily is widely popular. 
For a matrix ${\bf X}$ with unit $l_2$ norm columns, the  mutual coherence is defined as the maximum pair wise column correlation, i.e., 
\begin{equation}
\mu_{\bf X}=\underset{i\neq j}{\max}|{\bf X}_{i}^T{\bf X}_j|
\end{equation}
If $\mu_{\bf X}<\frac{1}{2k_0-1}$, then for all $k_0$-sparse vector $\boldsymbol{\beta}$, $erc({\bf X},\mathcal{I})$ can be bounded as $erc({\bf X},\mathcal{I})<\frac{k_0\mu_{\bf X}}{1-(k_0-1)\mu_{\bf X}}<1$\cite{tropp2004greed}. Hence,  $\mu_{\bf X}<\frac{1}{2k_0-1}$ is a sufficient condition for both noiseless and noisy sparse recovery using OMP. It is also shown that  $\mu_{\bf X}<\frac{1}{2k_0-1}$ is a worst case necessary condition for  sparse recovery.
\subsection{ Stopping conditions for OMP}
Most of the theoretical properties of OMP are derived assuming either the absence of noise\cite{tropp2004greed,tropp2007signal,extra} or the $\textit{a priori}$ knowledge of $k_0$\cite{OMP_wang}. In this case OMP iterations are terminated once $k=k_0$ or $\|{\bf r}^{(k)}\|_2=0$. When $k_0$ is not available which is typically the case, one has to rely on stopping conditions based on the properties of residual ${\bf r}^{(k)}$. For example, OMP can be stopped if $\|{\bf r}^{(k)}\|_2<\sigma\sqrt{n+2\sqrt{n\log(n)}}$ \cite{cai2011orthogonal,omp_rip_noise} or $\|{\bf X}^T{\bf r}^{(k)}\|_{\infty}<\sigma \sqrt{2\log(p)}$\cite{cai2011orthogonal}.  Likewise, \cite{resdif}  suggested a SC based on the  residual difference ${\bf r}^{(k)}-{\bf r}^{(k-1)}$. The necessary and sufficient conditions for high SNR consistency $(PE \rightarrow 0\  \text{as} \ \ \sigma^2 \rightarrow 0)$ of OMP with  residual based SC is derived in \cite{cs_tsp}. A generalized likelihood ratio based  stopping rule is developed in \cite{Xiong2014}. In addition to the subjectivity involved in the choice of SC,  all the above mentioned SC requires the knowledge of $\sigma^2$. As explained before, estimating $\sigma^2$ in underdetermined regression models is extremely difficult. In the following, we use the shorthand OMP($k_0$) for OMP with \textit{a priori} knowledge of $k_0$ and OMP($\sigma^2)$  for OMP with SC based on \textit{a priori} knowledge of $\sigma^2$.  In the next section, we develop TF-OMP, an OMP based procedure which does not require the knowledge of either $k_0$ or $\sigma^2$ for good performance. 

\section{Tuning Free orthogonal Matching Pursuit.} 
In this section, we present the proposed TF-OMP algorithm. This algorithm is based on the statistic $t(k)=\dfrac{\|{\bf r}^{(k)}\|_2^2}{\|{\bf r}^{(k-1)}\|_2^2}$, where ${\bf r}^{(k)}=({\bf I}_n-{\bf P}_{\mathcal{J}^k}){\bf y}$ is the residual in the $k^{th}$ iteration of OMP.  Using the property ${\bf P}_{\mathcal{J}^k}{\bf P}_{\mathcal{J}^{k-1}}{\bf y}={\bf P}_{\mathcal{J}^{k-1}}{\bf P}_{\mathcal{J}^{k}}{\bf y}={\bf P}_{\mathcal{J}^{k-1}}{\bf y}$ of projection matrices\cite{tsp}, we have $({\bf I}_n-{\bf P}_{\mathcal{J}^{k}})({\bf P}_{\mathcal{J}^{k}}-{\bf P}_{\mathcal{J}^{k-1}})={\bf O}_{n}$, where ${\bf O}_n$ is the $n \times n$ zero matrix. This implies that $\|({\bf I}_n-{\bf P}_{\mathcal{J}^{k-1}}){\bf y}\|_2^2= \|({\bf I}_n-{\bf P}_{\mathcal{J}^{k}}){\bf y}\|_2^2+ \|({\bf P}_{\mathcal{J}^{k}}-{\bf P}_{\mathcal{J}^{k-1}}){\bf y}\|_2^2$.  Hence, $t(k)$ can be rewritten as 
\begin{equation}
t(k)=\dfrac{\|({\bf I}_n-{\bf P}_{\mathcal{J}^{k}}){\bf y}\|_2^2}{\|({\bf I}_n-{\bf P}_{\mathcal{J}^{k}}){\bf y}\|_2^2+ \|({\bf P}_{\mathcal{J}^{k}}-{\bf P}_{\mathcal{J}^{k-1}}){\bf y}\|_2^2}
\end{equation}
Since the residual norms are non decreasing, i.e., $\|{\bf r}^{(k+1)}\|_2\leq \|{\bf r}^{(k)}\|_2$, we always have $0\leq t(k) \leq 1$. This statistic exhibits an interesting behaviour which is the core of our proposed technique, i.e., TF-OMP. Consider running OMP for a number of iterations ${k_{max}}>k_0$ such that neither the matrices $\{{\bf X}_{\mathcal{J}^k}\}_{k=1}^{k_{max}}$ are rank deficient nor the residuals $\{{\bf r}^{(k)}\}_{k=0}^{k_{max}}$ are  zero. Then $t(k)$ varies in the following manner for $1\leq k\leq  k_{max}$.

{\bf Case 1:-)}. {\bf When ${\mathcal J}^{k} \subset \mathcal{I}$:-} Then both $({\bf I}_n-{\bf P}_{\mathcal{J}^{k}}){\bf y}$ and $({\bf P}_{\mathcal{J}^{k}}-{\bf P}_{\mathcal{J}^{k-1}}){\bf y}$ contains contributions from signal ${\bf X\boldsymbol{\beta}}$ and noise ${\bf w}$. Since both numerator and denominator contains noise and signal terms, it is less likely that $t(k)$ takes very low values. \\
{\bf Case 2)}.{\bf When ${\mathcal J}^{k} \supseteq \mathcal{I}$ for the first time:-} In this case $({\bf I}_n-{\bf P}_{\mathcal{J}^{k}}){\bf X\boldsymbol{\beta}}={\bf 0}_n$ and $({\bf P}_{\mathcal{J}^{k}}-{\bf P}_{\mathcal{J}^{k-1}}){\bf X\boldsymbol{\beta}}\neq {\bf 0}_n$. Hence,
numerator has contribution only from the noise ${\bf w}$, whereas, denominator has contributions from both noise and signal ${\bf X\boldsymbol{\beta}}$. Hence, if signal strength is sufficiently high or noise level is low,  $t(k)$ will take very low values. \\
{\bf Case 3:-} {\bf When ${\mathcal J}^{k} \supset \mathcal{I}$:- } In this case both $({\bf I}_n-{\bf P}_{\mathcal{J}^{k}}){\bf X\boldsymbol{\beta}}={\bf 0}_n$ and $({\bf P}_{\mathcal{J}^{k}}-{\bf P}_{\mathcal{J}^{k-1}}){\bf X\boldsymbol{\beta}}={\bf 0}_n$. This means that  both numerator and denominator consists only of noise terms and hence the ratio $t(k)$ will not take very small value even if noise variance is very low. 

To summarize, as SNR improves, the minimal value of $t(k)$ for $1\leq k\leq k_{max}$ will  corresponds to that value of $k$ such that $\mathcal{I}\subseteq \mathcal{J}^{k}$ for the first time with a very high probability.  This point is illustrated in Fig.1  where a typical realization of the  quantity $t(k)$ is plotted for a matrix signal pair $({\bf X},\boldsymbol{\beta})$ satisfying ERC. The signal $\boldsymbol{\beta}$ has non zero values $\pm 1$ and $k_0=3$. At both SNR=10 dB and SNR=30 dB, the minimum value is attained at $k=3$  which is also the first time $\mathcal{I}\subseteq \mathcal{J}^{k}$. Further, the dip in the value of $t(k)$ at $k=3$ becomes more and more pronounced as SNR increases. This motivate the TF-OMP algorithm given in TABLE \rom{2} which try to estimate $\mathcal{I}$ by utilizing the sudden dip in $t(k)$.  
\begin{table}\centering
\label{tab:tf-omp}
\begin{tabular}{|l|}
\hline
{\bf Input:-} Observation ${\bf y}$, design matrix ${\bf X}$. \\
{\bf Step 1:-} Run OMP for $k_{max}$ iterations.\\ 
{\bf Step 2:-} Estimate $k^*=\underset{1\leq k\leq k_{max}-1}{\arg\min} t(k)$. \\

{\bf Step 3:-} Estimate support as $\hat{\mathcal{I}}=\mathcal{J}^{k_{max}}(1:k^*)$. \\  \ \ \ \ \ \ \ \ \ \ \ \ Estimate $\boldsymbol{\beta}$ as $\hat{\boldsymbol{\beta}}(\hat{\mathcal{I}})={\bf X}_{\hat{\mathcal{I}}}^{\dagger}{\bf y}$ and $\hat{\boldsymbol{\beta}}(\hat{\mathcal{I}}^C)={\bf 0}_{p-k^*}$.\\
{\bf Output:-} Support estimate $\hat{\mathcal{I}}$ and signal estimate $\hat{\boldsymbol{\beta}}$. \\
\hline
\end{tabular}
\caption{ Tuning free orthogonal matching pursuit}
\end{table}
We now make the following observations about TF-OMP.

\begin{figure}[htb]
\includegraphics[width=\columnwidth]{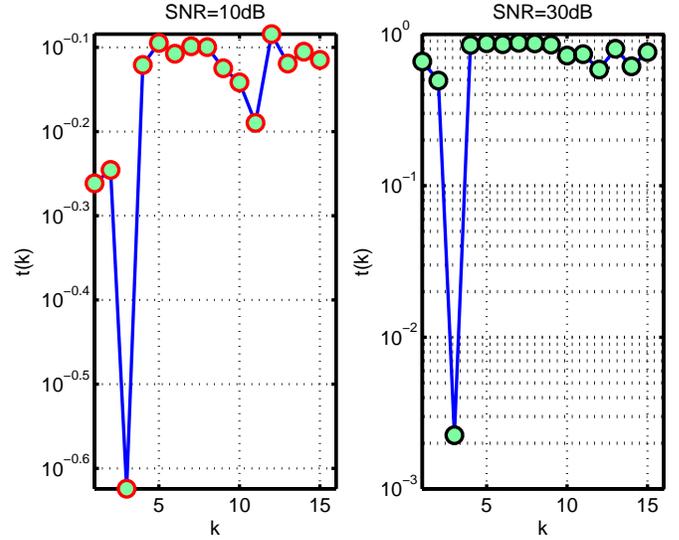}
\caption{ Variation of $t(k)$ vs k for the $32 \times 64$ matrix ${\bf X}=[{\bf I}_n,{\bf H}_n]$ described in Section \rom{6} when $k_0=3$.  }
\label{tab:tf-omp}
\end{figure}

\begin{remark}
It is important to note that the TF-OMP is designed to post facto estimate $k_f$, the first $k$ such that such that $\mathcal{I}\subseteq \mathcal{J}^{k}$ from a sequence of $t(k)$  related to OMP. Note that $k_f$ will correspond to $k_0$ only when the first $k_0$ iterations are accurate, i.e., at each of the first $k_0$ iterations, indices belonging to $\mathcal{I}$ are selected by OMP. Only in that situation will the objective of exact support recovery matches the objective of TF-OMP. When  conditions like RIC, MIC, ERC etc. are satisfied and SNR is high, it is established in \cite{cai2011orthogonal,OMP_wang,omp_rip_noise} that the first $k_0$ iterations are correct with a high probability. Under such circumstances, TF-OMP is trying to estimate $\mathcal{I}$ directly. 
\end{remark}
\begin{remark}
Next consider the situation where  $k_0\neq k_f$, i.e., the first $k_0$ iterations are not accurate. This situation happens in coherent design matrices  at all  SNR and incoherent dictionaries at low SNR. In this situation, all  versions of OMP including OMP($k_0$) fail to deliver accurate support recovery. Indeed, OMP($k_0$) results in missed discoveries (i.e., failure to include non zero entries of $\boldsymbol{\beta}$ in $\hat{\mathcal{I}}$) which cause flooring of MSE as SNR improves. TF-OMP has a qualitatively different behaviour. Since TF-OMP is trying to estimate $k_f$, it will produce a support estimate $\hat{\mathcal{I}}\supset \mathcal{I}$ provided that $\exists  \ k_f>k_0$ that satisfies $\mathcal{I} \subset \mathcal{J}^{k_f}$. Such delayed recovery happens quiet often in coherent dictionaries\cite{extra}. In other words, TF-OMP has a lesser tendency to have missed discoveries, rather it suffers from false discoveries (including non significant indices in $\hat{\mathcal{I}}$). This tendency can result in a degraded MSE performance for TF-OMP at low SNR. However, as SNR improves the effect of false discoveries on MSE decreases, whereas, the effect of missed discoveries become more predominant. Consequently, TF-OMP  suffer less from MSE floors in such situations than OMP($k_0$). To summarise, when there is no congruency between $k_0$ and $k_f$, TF-OMP can potentially deliver better MSE performance than OMP($k_0$) at least in high SNR.  
\end{remark}
\begin{remark} The only user defined parameter in TF-OMP is $k_{max}$. This can be set independent of the signal $\boldsymbol{\beta}$. The only requirement for efficient operation of TF-OMP is that $k_{max}>k_0$, $\{{\bf X}_{\mathcal{J}^k}\}_{k=1}^{k_{max}}$ are full rank and the residuals $\{{\bf r}^{(k)}\}_{k=0}^{k_{max}}$ are not zero.  It is impossible to ascertain \textit{a priori} when the matrices become rank deficient or residuals become zero. Hence, one can set $k_{max}=n$ (since ${\bf X}_{\mathcal{J}^{n+1}}$ is rank deficient w.p.1) initially and terminate iterations when any of the aforementioned contingencies happen. However, the maximum value of $k_0$ for a fixed $n$ that can be recovered using any sparse recovery technique is $\floor{\dfrac{n+1}{2}}$. This follows from the fact that Spark of ${\bf X}$ satisfies $spark({\bf X})\leq n+1$ (equality for equiangular tight frames) and $k_0<\floor{\dfrac{spark({\bf X})}{2}}$ is  a necessary  condition for sparse recovery using any algorithm\cite{elad_book}.  Hence, instead of $k_{max}=n$, it is sufficient to set $k_{max}=\floor{\dfrac{n}{2}}$ and this is the value of $k_{max}$  used in our simulations. Needless to say, if one has \textit{a priori} knowledge of maximum value of $k_0$ (not the exact value of $k_0$), $k_{max}$ can be set to that value also.  
\end{remark}
\subsection{Computational complexity of TF-OMP}
The computational complexity of TF-OMP with $k_{max}=\floor{\dfrac{n}{2}}$ is $O(n^2p)$  which is higher than the $O(npk_0)$ complexity  of OMP($k_0$).   This is the cost one has to pay for not knowing $k_0$ or $\sigma^2$ \textit{a priori}.  However, TF-OMP is computationally much more efficient than either the second order conic programming (SOCP) or cyclic algorithm based implementation of the popular tuning free SPICE algorithm\cite{spice_like}. Even the cyclic algorithm based implementation of SPICE which is claimed to be computationally efficient (in comparison with SOCP) in small and medium sized problems involve multiple iterations and in each iteration it requires the inversion of a  $n \times n$ matrix ($O(n^3)$ complexity) and a matrix matrix multiplication of complexity $O((n+p)^{3})$. It is possible to reduce the complexity of TF-OMP by producing upper bounds on $k_0$ that is lower than the $k_{max}=\floor{\dfrac{n}{2}}$ used in TF-OMP.  Assuming \textit{a priori} knowledge of an upper bound $k_{up} \geq  k_0$ is a significantly weaker assumption than having exact \textit{a priori} knowledge of $k_0$. If one can produce an upper bound $k_{up}\geq k_0$ satisfying   $k_{up}=O(k_0)$, then setting $k_{max}=k_{up}$  in TF-OMP gives the OMP($k_0$) complexity of $O(k_0np)$.  

For situations where the statistician is completely oblivious to $k_0$, we propose two low complexity versions of TF-OMP, \textit{viz.}, QTF-OMP1 (quasi tuning free OMP) and QTF-OMP2 that uses a  value of $k_{max}$ lower than the $k_{max}=\floor{\dfrac{n}{2}}$ used in TF-OMP.  QTF-OMP1 uses $k_{max}=1+\floor{\sqrt{\dfrac{n(p-1)}{p-n}}}$ and QTF-OMP2 uses $k_{max}=\floor{n/\log(p)}$. QTF-OMP1 is motivated by the fact that the best coherence based guarantee for OMP extends upto $k_0\leq \dfrac{1}{2}(1+\dfrac{1}{\mu_{\bf X}})$ and $\mu_{\bf X}$ for any $n \times p$ matrix satisfies $\mu_{\bf X}\geq \sqrt{\dfrac{p-n}{n(p-1)}}$\cite{elad_book}. Hence, QTF-OMP1 uses a value of $k_{max}$ which is two times higher than the maximum value of $k_0$ that can be covered by  the coherence based guarantees available for OMP. Likewise, the best known asymptotic guarantee for OMP states that OMP can recover any $k_0$ sparse signal when $(n,p,k_0) \rightarrow \infty$ if $n=2k_0(1+\delta)\log(p-k_0)$, where $\delta>0$ is any arbitrary value\cite{brownian}. Hence, when $p\gg k_0$, the highest value of $k_0$ one can reliably detect  using OMP asymptotically is $\dfrac{n}{2\log(p)}$. The value of $k_{max}$ used in QTF-OMP1 and QTF-OMP2 is twice of the aforementioned maximum detectable values of $k_0$ to add sufficient robustness. The complexity of  QTF-OMP1 and QTF-OMP2 are $O\left(np\sqrt{\dfrac{n(p-1)}{p-n}}\right)$ and $O\left(n^2\dfrac{p}{\log(p)}\right)$ which is significantly lower than the $O(n^2p)$ complexity  of TF-OMP. Unlike TF-OMP which is completely tuning free, QTF-OMP1 and QTF-OMP2 involves  a subjective choice of $k_{max}$ (though motivated by theoretical properties). The rest of this article consider TF-OMP only and in Section \rom{6} we demonstrate that the performance of TF-OMP, QTF-OMP1 and QTF-OMP2 are similar across multiple experiments.

\section{ Analysis of  TF-OMP} 
In this section we will mathematically analyse various factors that will influence the performance of TF-OMP. In particular we discuss the conditions for successful recovery of a $k_0$-sparse vector in $l_2$  bounded noise   $\|{\bf w}\|_2\leq \epsilon_2$. Note that the Gaussian vector ${\bf w}\sim \mathcal{N}({\bf 0}_n,\sigma^2{\bf I}_n)$ is  essentially bounded in the sense that $\mathbb{P}\left(\|{\bf w}\|_2>\sigma\sqrt{n+2\sqrt{n\log(n)}}\right)\leq \dfrac{1}{n}$. Hence with $\epsilon_2=\sigma\sqrt{n+2\sqrt{n\log(n)}}$, this analysis is applicable to Gaussian noise too.   For bounded noise, we define the SNR as $SNR_b=\dfrac{\|{\bf X}\boldsymbol{\beta}\|_2^2}{\epsilon_2^2}$.  We next state and  prove a theorem  regarding the  successful support recovery by TF-OMP in bounded noise. Note that the accurate support recovery automatically translate to a MSE performance equivalent to that of an oracle with \textit{a priori} knowledge of support $\mathcal{I}$. Throughout this section, we use $t(k)$ to denote the ratio ${\|{\bf r}^{(k)}\|_2}/{\|{\bf r}^{(k-1)}\|_2}$ instead of ${\|{\bf r}^{(k)}\|_2^2}/{\|{\bf r}^{(k-1)}\|_2^2}$. 

\begin{thm} For any matrix signal pair $({\bf X},\boldsymbol{\beta})$ satisfying ERC, MIC or RIP with $\delta_{k_0+1}\leq \dfrac{1}{\sqrt{k_0}+1}$, TF-OMP with $k_{max}>k_0$ will recover the correct support in bounded noise if the SNR ($SNR_b$) is sufficiently high. 
\end{thm}  
\emph{Proof:-} The analysis of TF-OMP is based on the fundamental results developed in the  \cite{cai2011orthogonal} and \cite{omp_rip_noise} stated next.
\subsection{ A brief review of relevant results from \cite{cai2011orthogonal} and \cite{omp_rip_noise}.}
 Let $\lambda_{min}$ and $\lambda_{max}$ denotes the minimum and maximum eigenvalues of ${\bf X}_{\mathcal{I}}^T{\bf X}_{\mathcal{I}}$ respectively. 
\begin{lemma} If $\lambda_{min}>0$ and $erc({\bf X},\mathcal{I})<1$, then the following  statements hold true\cite{cai2011orthogonal}.\\  
{\bf A1):-} $\lambda_{min}\|{\bf \boldsymbol{\beta}}_{u^k}\|_2-\epsilon_2\leq \|{\bf r}^{(k)}\|_2\leq \lambda_{max}\|{\bf \boldsymbol{\beta}}_{u^k}\|_2+\epsilon_2$, $1\leq k\leq k_0$. Here ${u^k}=\mathcal{I}/\mathcal{J}^k$  denotes the indices in $\mathcal{I}$ that are not selected after the $k^{th}$ iteration. \\
{\bf A2):-} $\epsilon_2\leq \dfrac{\boldsymbol{\beta}_{min}\lambda_{min}(1-erc({\bf X},\mathcal{I}))}{2}$ implies that the first $k_0$ iterations are correct, i.e., $\{t_1,\dots,t_{k_0}\}=\mathcal{I}$. 
 \end{lemma}
A1) shows how to bound the residual norms used in $t(k)$ based on $\lambda_{max}$ and $\lambda_{min}$. A2)  implies that  the first $k_0$ iterations of OMP will be correct if $SNR_b$ is sufficiently high and ERC  is satisfied. We now state conditions similar to A1)-A2) in terms of MIC and RIC.

\begin{lemma}
 MIC $\left(\mu_{\bf X} \leq \dfrac{1}{2k_0-1}\right)$ implies that $ erc({\bf X},\mathcal{I})\leq \dfrac{1-(k_0-1)\mu_{\bf X}}{1+(k_0-1)\mu_{\bf X}}<1$ and $0<1-(k_0-1)\mu_{\bf X}<\lambda_{min}\leq\lambda_{max}\leq 1+(k_0-1)\mu_{\bf X}$\cite{tropp2004greed}. Substituting these bounds in  A1) and A2) gives \\
{\bf B1):-} $\left(1-(k_0-1)\mu_{\bf X}\right)\|{\bf \boldsymbol{\beta}}_{u^k}\|_2-\epsilon_2\leq \|{\bf r}^{(k)}\|_2\leq \left(1+(k_0-1)\mu_{\bf X}\right)\|{\bf \boldsymbol{\beta}}_{u^k}\|_2+\epsilon_2$, $1\leq k\leq k_0$. \\
{\bf B2):-} $ \epsilon_2\leq \dfrac{\boldsymbol{\beta}_{min}(1-(2k_0-1)\mu_{\bf X})}{2}$ implies that the first $k_0$ iterations are correct, i.e., $\{t_1,\dots,t_{k_0}\}=\mathcal{I}$.
\end{lemma}

\begin{lemma}
RIC $\delta_{k_0}<1$ implies that $1-\delta_{k_0}\leq\lambda_{min}\leq \lambda_{max}\leq 1+\delta_{k_0}$\cite{omp_rip_noise}. Substituting this in A1) gives \\
{\bf C1):-} $(1-\delta_{k_0})\|{\bf \boldsymbol{\beta}}_{u^k}\|_2-\epsilon_2\leq \|{\bf r}^{(k)}\|_2\leq (1+\delta_{k_0})\|{\bf \boldsymbol{\beta}}_{u^k}\|_2+\epsilon_2$, $1\leq k\leq k_0$. \\
The next statement follows from Theorem 1 of \cite{omp_rip_noise}.\\
{\bf C2):-} If $\delta_{k_0+1}< \dfrac{1}{\sqrt{k_0}+1}$, then $\epsilon_2\leq \dfrac{\boldsymbol{\beta}_{min}(1-(\sqrt{k_0}+1)\delta_{k_0+1})}{\sqrt{1+\delta_{k_0+1}}+1}$ implies that the first $k_0$ iterations are correct.
\end{lemma}

Since the analysis based on $erc({\bf X},\mathcal{I})$ and $\{\lambda_{min},\lambda_{max}\}$ are more general than MIC or RIC,  we explain TF-OMP using $erc({\bf X},\mathcal{I})$ and $\{\lambda_{min},\lambda_{max}\}$. However, as outlined in B1)-B2) and C1)-C2), this analysis can be easily replaced by $\mu_{\bf X}$ and $\{\delta_{k_0},\delta_{k_0+1}\}$. 
 
\subsection{Sufficient conditions for sparse recovery using TF-OMP}
The successful recovery of support of $\boldsymbol{\beta}$ using TF-OMP requires the simultaneous occurrence of the events E1)-E3) given below.\\
E1). The first $k_0$ iterations are correct, i.e., $\{t_1,\dots,t_{k_0}\}=\mathcal{I}$. \\
E2). $t(k)>t(k_0),\  \text{for} \ 1\leq k\leq k_0-1$. \\
E3). $t(k)>t(k_0), \  \text{for} \ k_0+1\leq k\leq k_{max}$. \\
E1) implies that  OMP with \textit{a priori} knowledge of $k_0$, i.e., OMP($k_0)$ can perform exact sparse recovery, whereas, E2) and E3) implies that TF-OMP will be free from missed and false discoveries respectively. Note that the condition  A2)  implies that  the event E1) occurs as long as $\lambda_{min}>0$, $erc({\bf X},\mathcal{I})<1$ and $\epsilon_2$ is below a particular level $\epsilon^a$ given by
 \begin{equation}
 \epsilon^a=\dfrac{\boldsymbol{\beta}_{min}\lambda_{min}(1-erc({\bf X},\mathcal{I}))}{2}.
 \end{equation}
  Next we consider the events E2)  and E3) assuming that the noise ${\bf w}$ satisfies $\|{\bf w}\|_2 \leq \epsilon^a$, i.e., E1) is true. To establish $t(k_0)<t(k)$ for $k \neq k_0$, we produce an upper bound on $t(k_0)$ and lower bounds on $t(k)$ for $k \neq k_0$ and show that the upper bound on $t(k_0)$ is lower than the lower bound on $t(k)$ for $k\neq k_0$ at high SNR. We first consider the event E2). Since all $k_0$ entries in $\boldsymbol{\beta}$ are selected in the first $k_0$ iterations $u^{k_0}=\phi$  and hence $\|\boldsymbol{\beta}_{ u^{k_0}}\|_2=0$. Likewise, only one entry in $\boldsymbol{\beta}$ is left out after $k_0-1$ iterations. Hence, $|u^{k_0-1}|=1$ and  $\|\boldsymbol{\beta}_{u^{k_0-1}}\|_2\geq \boldsymbol{\beta}_{min}$. Substituting these values  in A1) of Lemma 1, we have $\|{\bf r}^{(k_0)}\|_2\leq \epsilon_2$ and $\|{\bf r}^{(k_0-1)}\|_2\geq \lambda_{min}\boldsymbol{\beta}_{min}-\epsilon_2$. Hence, $t(k_0)$ is  bounded by
\begin{equation}\label{ub_tk0}
t(k_0)=\dfrac{\|{\bf r}^{(k_0)}\|_2}{\|{\bf r}^{(k_0-1)}\|_2} \leq \dfrac{\epsilon_2}{\lambda_{min}\boldsymbol{\beta}_{min}-\epsilon_2}, \forall \epsilon_2<\epsilon^a.
\end{equation}
Next we lower bound $t(k)$ for $k<k_0$. Note that $\boldsymbol{\beta}_{u^{k-1}}=\boldsymbol{\beta}_{u^{k}}+\boldsymbol{\beta}_{u^{k-1}/u^{k}}$
after appending enough zeros in appropriate locations. Further, $\|\boldsymbol{\beta}_{u^{k-1}/u^{k}}\|_2=|\boldsymbol{\beta}_{u^{k-1}/u^{k}}|\leq \boldsymbol{\beta}_{max}$, where $\boldsymbol{\beta}_{max}=\underset{k \in \mathcal{I}}{\max}|\boldsymbol{\beta}_k|$. Applying triangle inequality to $\boldsymbol{\beta}_{u^{k-1}}=\boldsymbol{\beta}_{u^{k}}+\boldsymbol{\beta}_{u^{k-1}/u^{k}}$ gives
the bound 
\begin{equation}\label{temp_bound}
\|\boldsymbol{\beta}_{u^{k-1}}\|_2\leq \|\boldsymbol{\beta}_{u^{k}}\|_2+\|\boldsymbol{\beta}_{u^{k-1}/u^{k}}\|_2 \leq  \|\boldsymbol{\beta}_{u^{k}}\|_2+  \boldsymbol{\beta}_{max}
\end{equation}
Applying (\ref{temp_bound}) in $t(k)$ gives
\begin{equation}\label{A1bound}
\begin{array}{ll}
t(k)=\dfrac{\|{\bf r}^{(k)}\|_2}{\|{\bf r}^{(k-1)}\|_2} &\geq \dfrac{\lambda_{min}\|\boldsymbol{\beta}_{u^k}\|_2-\epsilon_2}{\lambda_{max}\|\boldsymbol{\beta}_{u^{k-1}}\|_2+\epsilon_2}\\
&\geq \dfrac{\lambda_{min}\|\boldsymbol{\beta}_{u^k}\|_2-\epsilon_2}{\lambda_{max}(\|\boldsymbol{\beta}_{u^{k}}\|_2+\boldsymbol{\beta}_{max})+\epsilon_2}\\
\end{array}
\end{equation}
for $k<k_0$ and $\epsilon_2\leq \epsilon_a$. The R.H.S of (\ref{A1bound}) can be rewritten as
\begin{equation}\label{A1bound2}
\dfrac{\lambda_{min}\|\boldsymbol{\beta}_{u^k}\|_2-\epsilon_2}{\lambda_{max}(\|\boldsymbol{\beta}_{u^{k}}\|_2+\boldsymbol{\beta}_{max})+\epsilon_2}=\dfrac{\lambda_{min}}{\lambda_{max}}\left(1-\dfrac{\dfrac{\epsilon_2}{\lambda_{min}}+\dfrac{\epsilon_2}{\lambda_{max}}+\boldsymbol{\beta}_{max}}{\|\boldsymbol{\beta}_{u^{k}}\|_2+\boldsymbol{\beta}_{max}+\dfrac{\epsilon_2}{\lambda_{max}}}\right)
\end{equation}
From (\ref{A1bound2}) it is clear that the R.H.S of (\ref{A1bound}) decreases with decreasing $\|\boldsymbol{\beta}_{u^k}\|_2$.  Note  that the minimum value of $\|\boldsymbol{\beta}_{u^{k}}\|_2$ is $\boldsymbol{\beta}_{min}$ itself. This leads to an even smaller lower bound on $t(k)$ for $k<k_0$ given by
\begin{equation}\label{lb_on_klessk0}
t(k)\geq \dfrac{\lambda_{min}\boldsymbol{\beta}_{min}-\epsilon_2}{\lambda_{max}(\boldsymbol{\beta}_{max}+\boldsymbol{\beta}_{min})+\epsilon_2},\  \ \forall k<k_0\  \text{and} \  \epsilon_2<\epsilon^a.
\end{equation}
For E2) to happen it is sufficient that the lower bound on $t(k)$ for $k<k_0$ is larger than the upper bound on $t(k_0)$, i.e., 
\begin{equation}
\dfrac{\lambda_{min}\boldsymbol{\beta}_{min}-\epsilon_2}{\lambda_{max}(\boldsymbol{\beta}_{max}+\boldsymbol{\beta}_{min})+\epsilon_2} \geq  \dfrac{\epsilon_2}{\lambda_{min}\boldsymbol{\beta}_{min}-\epsilon_2}.
\end{equation}
This will happen if $\epsilon_2\leq \epsilon^b$, where 
\begin{equation}\label{epsb}
\epsilon^b= \dfrac{\lambda_{min}\boldsymbol{\beta}_{min}}{1+2\dfrac{\lambda_{max}}{\lambda_{min}} + \dfrac{\lambda_{max}}{\lambda_{min}} \dfrac{\boldsymbol{\beta}_{max}}{\boldsymbol{\beta}_{min}} }. 
\end{equation}
In  words, whenever $\epsilon_2\leq \min(\epsilon^a,\epsilon^b)$, TF-OMP will not have any missed discoveries. 

Next we consider the event E3) and assume again that $\epsilon_2<\epsilon^a$. Since, the first $k_0$ iterations are correct, $({\bf I}_n-{\bf P}_k){\bf y}=({\bf I}_n-{\bf P}_k){\bf w}$ for $k\geq k_0$.  Note that the quantity $t(k)=\dfrac{\|({\bf I}-{\bf P}_k){\bf w}\|_2}{\|({\bf I}-{\bf P}_{k-1}){\bf w}\|_2}$  is independent of the scaling factor $\epsilon_2$.  Hence, define the quantity 
\begin{equation}
\Gamma({\bf X}, \mathcal{A})=\underset{\underset{  \|{\bf z}\|_2=1}{k_0<k\leq k_{max}} }{\min}\dfrac{\|({\bf I}-{\bf P}_k){\bf z}\|_2}{\|({\bf I}-{\bf P}_{k-1}){\bf z}\|_2}.
\end{equation}
where $\mathcal{A}=\{t_{k_0},\dotsc, t_{k_{max}}\}\subset\{1,\dotsc,p\}$ is an ordered set representing the indices selected by OMP. By the definition of $\Gamma({\bf X}, \mathcal{A})$,  $t(k)\geq \Gamma({\bf X},\mathcal{A}),\forall k>k_0$.  $\Gamma({\bf X},\mathcal{A})$ is a random variable  depending on the indices $\mathcal{A}=\{t_{k_0},\dotsc,t_{k_{max}}\}$ which depends  on the noise vector ${\bf w}$. However, ${\bf w}$ influences  $\Gamma({\bf X},\mathcal{A})$ only through $\mathcal{A}$. Since  $\{{\bf P}_k\}_{k=k_0+1}^{k_{max}}$ depends on ${\bf w}$, it is difficult to characterize $\Gamma({\bf X},\mathcal{A})$. TF-OMP stops before ${\bf r}^{(k)}={\bf 0}_n$ deterministically and hence  it is true that $\Gamma({\bf X},\mathcal{A})>\gamma_{\mathcal{A}}>0$ for each of the possible realization of ${\bf w}$ or equivalently, each possible realization $\mathcal{A}$.  Further, the set  of all possible $\mathcal{A}$ denoted by $\tilde{\mathcal{A}}$ is large, but finite. This implies that $\Gamma({\bf X},\mathcal{I})=\underset{\mathcal{A} \in \tilde{\mathcal{A}}}{\min}\Gamma({\bf X},\mathcal{A})\geq \underset{\mathcal{A} \in \tilde{\mathcal{A}}}{\min}\gamma_{\mathcal{A}}>0$. This implies that $t(k)\geq \Gamma({\bf X},\mathcal{I})>0$ with probability one for all $k>k_0$ and $\Gamma({\bf X},\mathcal{I})$ is independent of $\epsilon_2$.  At the same time, the bound $t(k_0)\leq \dfrac{\epsilon_2}{\lambda_{min}\boldsymbol{\beta}_{min}-\epsilon_2}$ on $t(k_0)$ decreases to zero with decreasing $\epsilon_2$. Hence, $\exists \epsilon^c>0$ given by
\begin{equation}\label{epsilon_c}
\epsilon^c=\dfrac{\lambda_{min}\boldsymbol{\beta}_{min}\Gamma({\bf X},\mathcal{I})}{1+\Gamma({\bf X},\mathcal{I})}
\end{equation} 
such that   $t(k_0)<t(k)$ for all $k>k_0$ whenever $\epsilon_2<\min(\epsilon^a,\epsilon^c)$. In words, TF-OMP will not make false discoveries whenever $\epsilon_2<\min(\epsilon^a,\epsilon^c)$. Combining all the required conditions, we can see that TF-OMP will recover the correct support whenever $0<\epsilon_2<\epsilon_{min}=min(\epsilon^a,\epsilon^b,\epsilon^c)$. In words, for any support $\mathcal{I}$ satisfying ERC,  $\exists SNR_b^{\mathcal{I}}<\infty$, such that TF-OMP will recover $\mathcal{I}$ whenever $SNR_b>SNR_b^{\mathcal{I}}$. Hence proved. $\blacksquare$  \\
We now make some  remarks about the performance of TF-OMP.
\begin{remark}
The  conditions on the matrix support pair $({\bf X},\mathcal{I})$ for the successful support recovery using TF-OMP is exactly same as the MIC, ERC and RIC based conditions outlined for OMP($\sigma^2$) and OMP($k_0$). From the expressions of $\epsilon_b$ in (\ref{epsb}) and $\epsilon_c$ in (\ref{epsilon_c}), it is difficult to ascertain whether the $min(\epsilon_b,\epsilon_c)<\epsilon_a$.  In other words, it is difficult to state whether the required SNR for successful recovery using TF-OMP is higher than that required for OMP($k_0$) or OMP($\sigma^2$). However, extensive numerical simulations indicate that except in  the very low SNR regime, TF-OMP performs very closely compared to OMP($k_0$).  This comparatively poor performance at low SNR can be directly attributed to the lack of knowledge of $k_0$ or $\sigma^2$.  Note that the analysis in this article is worst case and qualitative in nature. Deriving exact conditions on $\epsilon_2$ for successful recovery will be more difficult and is not pursued in this article.   
\end{remark}
\begin{remark}
The bound (\ref{epsb}) involves the term $\dfrac{\boldsymbol{\beta}_{max}}{\boldsymbol{\beta}_{min}}$ in the denominator. In particular, (\ref{epsb}) implies that the noise level that allows for successful recovery, i.e., $\epsilon_{min}$ decreases with increasing $\dfrac{\boldsymbol{\beta}_{max}}{\boldsymbol{\beta}_{min}}$.  This is the main qualitative difference between TF-OMP and the  results in \cite{cai2011orthogonal} and \cite{wang2012recovery} available for OMP($\sigma^2$) and OMP($k_0$).  This term can be attributed to the sudden fall in the residual power when a ``very significant" entry in $\boldsymbol{\beta}$ is covered by OMP at an intermediate iteration which mimic the fall in residual power when the  ``last" entry in $\boldsymbol{\beta}$ is selected in the $k_0$ iteration. Note that the later fall in residual power is what TF-OMP trying to detect. The main implication of this result is that the TF-OMP will be lesser effective while recovering $\boldsymbol{\beta}$ with significant variations (high $\dfrac{\boldsymbol{\beta}_{max}}{\boldsymbol{\beta}_{min}}$ ratio) than in recovering signals with lesser variations. 
\end{remark}
\subsection{High SNR consistency of TF-OMP in Gaussian noise. }
The high SNR consistency of variable selection techniques in Gaussian noise has received considerable attention in signal processing community recently\cite{ding2011inconsistency,tsp,spl,cs_tsp}. High SNR consistency is formally defined as follows. \\
{\bf Definition:-} A support recovery technique is high SNR consistent iff
 the probability of support recovery error (PE) satisfies $\underset{\sigma^2 \rightarrow 0 }{\lim}PE=0$.  \\
The following lemma stated and proved in \cite{cs_tsp} establish the necessary and sufficient condition for the high SNR consistency of OMP and LASSO. 
 \begin{lemma}\label{cs_lemma}
 LASSO in (\ref{lasso}) is high SNR consistent for any matrix signal pair $({\bf X},\boldsymbol{\beta})$ satisfying ERC if $\underset{\sigma^2\rightarrow 0}{\lim}{\lambda}=0$ and $\underset{\sigma^2\rightarrow 0}{\lim}{\dfrac{\lambda}{\sigma}}=\infty$. OMP with SC that terminate iterations when $\|{\bf r}^{(k)}\|_2\leq \gamma$ or $\|{\bf X}^T{\bf r}^{(k)}\|_{\infty}\leq \gamma$ are high SNR consistent if $\underset{\sigma^2 \rightarrow 0}{\lim}\gamma=0$ and $\underset{\sigma^2 \rightarrow 0}{\lim}\dfrac{\gamma}{\sigma}=\infty$. 
 \end{lemma}
 Lemma \ref{cs_lemma} implies that LASSO and OMP with residual based SC are high SNR consistent iff the tuning parameters are adapted according to $\sigma^2$. In particular, Lemma \ref{cs_lemma} implies that widely used parameters for LASSO like $\lambda= \sigma\sqrt{2\log(p)}$ in \cite{candes2009near} and OMP SC with $\gamma=\sigma\sqrt{n+2\sqrt{n\log(n)}}$ are inconsistent at high SNR. In the following theorem, we state and prove the high SNR consistency of TF-OMP. To the best of our knowledge no CS algorithm is shown to achieve high SNR consistency in the absence of knowledge of $\sigma^2$. 
 \begin{thm}TF-OMP is high SNR consistent for any matrix signal pair  $({\bf X},\boldsymbol{\beta})$ that satisfy ERC. 
 \end{thm}
 \begin{proof}
 From the analysis of Section \rom{4}-B, we know that TF-OMP recover the correct support whenever $\|{\bf w}\|_2< \epsilon_{min}$, where $\epsilon_{min}>0$ is a function of $\boldsymbol{\beta}_{min},\boldsymbol{\beta}_{max}$ and support $\mathcal{I}$. Hence, $\mathbb{P}(\hat{\mathcal{I}}=\mathcal{I})$ satisfies $\mathbb{P}(\hat{\mathcal{I}}=\mathcal{I})\geq \mathbb{P}(\|{\bf w}\|_2^2\leq \epsilon_{min}^2)=\mathbb{P}(\dfrac{\|{\bf w}\|_2^2}{\sigma^2}\leq \dfrac{\epsilon_{min}^2}{\sigma^2}) $. Note that $T=\dfrac{\|{\bf w}\|_2^2}{\sigma^2}\sim \chi^2_n$. Also the distribution of  $T$ is independent of $\sigma^2$. Further, $T$ is bounded in probability in the sense that  $\underset{l \rightarrow \infty}{\lim}\mathbb{P}(T<a_l)=1$ for any sequence $a_l \rightarrow \infty$. All these implies that  
 \begin{equation}
 \underset{\sigma^2 \rightarrow 0}{\lim}\mathbb{P}(\hat{\mathcal{I}}= \mathcal{I})\geq \underset{\sigma^2 \rightarrow 0}{\lim}\mathbb{P}(T<\dfrac{\epsilon_{min}^2}{\sigma^2})=1.
 \end{equation}
 Hence proved.
 \end{proof}
\section{ Tuning Free Robust Linear Regression in the presence of sparse outliers} 
Throughout this article we have considered a linear regression model ${\bf y}={\bf X}\boldsymbol{\beta}+{\bf w}$ where $\boldsymbol{\beta}$ is a sparse vector and ${\bf w} \sim \mathcal{N}({\bf 0}_n,\sigma^2{\bf I}_n)$ is the noise.  In this section we consider a different regression model
\begin{equation}
{\bf y}={\bf X}\boldsymbol{\beta}+{\bf w}+{\bf g}
\end{equation} 
where ${\bf X} \in \mathbb{R}^{n\times p}$ is a full rank design matrix with $n>p$ or $n\gg p$, regression vector $\boldsymbol{\beta}$ may or may not be sparse and the inlier noise ${\bf w} \sim \mathcal{N}({\bf 0}_n,\sigma^2{\bf I}_n)$. The outlier noise ${\bf g}$ represents the large   errors in the regression equation that is not modelled using the inlier noise distribution. In addition to SNR,  this regression model also require signal to interference ratio  (SIR) given by
\begin{equation}
\text{SIR}=\dfrac{\|{\bf X\boldsymbol{\beta}}\|_2^2}{\|{\bf g}\|_2^2}.
\end{equation} 
to quantify the impact of outliers. In many cases of practical interest ${\bf g}$ is modelled as sparse, i.e., $n_{out}=|supp({\bf g})|\ll n$. However, ${\bf g}$ can have very large power, i.e., SIR can be very low\cite{error_correction_candes,bdrao_robust,maronna2006wiley,koushik}.  A classic example of this is channel estimation in OFDM systems in the presence of narrow band interference\cite{NBI}.  In spite of the full rank of ${\bf X}$, traditional least squares (LS) estimate of $\boldsymbol{\beta}$ given by $\hat{\bf \boldsymbol{\beta}}={\bf X}^{\dagger}{\bf y}$ is highly inefficient in terms of MSE. An algorithm called greedy algorithm for robust de-noising (GARD)\cite{gard}  which is very closely related to the OMP algorithm discussed in this paper was proposed in \cite{gard} for such scenarios. An algorithmic description of GARD is given in TABLE \ref{tab:gard}. 
\begin{table}\centering
\begin{tabular}{|l|}
\hline
{\bf Input:-} Observed vector ${\bf y}$, Design Matrix ${\bf X}$ and SC. \\
{\bf Initialization:-} ${\bf A}^{(0)}={\bf X}$, ${\bf r}^{(0)}=({\bf I}_n-{\bf P}_{{\bf A}^{(0)}}){\bf y}$. k=1. \\
Repeat Steps 1-4 until SC is met. \\
{\bf Step 1:-} Identify the strongest residual in ${\bf r}^{(k-1)}$, i.e., \\   \ \ \ \ \ \ \ \ \ \ ${j}_k=\underset{j=1,\dotsc,n}{\arg\max}|{\bf r}^{(k-1)}_j|$. \\
{\bf Step 2:-} Update the matrix ${\bf A}^{(k)}=[{\bf A}^{(k-1)}\ \ \ {\bf e}_{{j}_k}]$.\\
{\bf Step 3:-} Jointly estimate $\boldsymbol{\beta}$ and ${\bf g}_{j_1},\dotsc,{\bf g}_{j_k}$ as \\
\ \ \ \ \ \ \ \ \ \ $
[{\hat{\boldsymbol{\beta} }}^T  [\hat{\bf g}_{j_1},\dotsc,\hat{\bf g}_{j_k}] ]^T= {{\bf A}^{(k)}}^{\dagger}{\bf y}.
$\\
{\bf Step 4:-} Update the residual ${\bf r}^{(k)}=({\bf I}_n-{\bf P}_{{\bf A}^{(k)}}){\bf y}$. \\
\ \ \ \ \ \ \ \  \ \ \ \             $k \leftarrow k+1$. \\
 {\bf Output:-} Signal estimate $\hat{\boldsymbol{\beta}}$.\\
\hline
\end{tabular}
\caption{Greedy algorithm for robust de-noising. ${\bf e}_k \in \mathbb{R}^n$ is the $k^{th}$ column of $n \times n$ identity matrix ${\bf I}_n$. }
\label{tab:gard}
\end{table}

GARD can be considered as applying OMP to identify the significant entries in ${\bf g}$  after nullifying the signal component ${\bf X\boldsymbol{\beta}}$ in the regression equation by projecting ${\bf y}$ onto a subspace orthogonal to the column span of ${\bf X}$. Just like OMP, the key component in GARD is the SC. One can stop GARD when $n_{out}$ (which is unknown \textit{a priori}) iterations are performed or when the residual $\|{\bf r^{(k)}}\|_2$ falls below a predefined threshold. However, setting the threshold requires the knowledge of $\sigma^2$. We use the shorthand GARD($n_{out}$) and GARD($\sigma^2$) to represent these schemes. However, producing high quality estimate of $\sigma^2$ in the presence of outliers is also a difficult task.  Further, there exists a level of subjectivity in the choice of this threshold even if $\sigma^2$ is known. A better strategy would be to produce a version of GARD free of any tuning parameters. 

The principle developed for TF-OMP can be used in GARD also. To explain  this, consider the statistic $t(k)=\dfrac{\|{\bf r}^{(k)}\|_2^2}{\|{\bf r}^{(k-1)}\|_2^2}$ and let $k_f$ be the first iteration at which $supp({\bf g}) \subseteq [j_1,\dotsc,j_{k_f}]$. For all $k<k_f$, ${\bf r}^{(k)}$ contains contributions from the outlier ${\bf g}$, whereas for all $k>k_g$, ${\bf r}^{(k)}$ has contributions from noise only. Hence, if entries in ${\bf g}$ are sufficiently large in comparison with noise level $\sigma^2$, just like in the case of OMP, $t(k)$ experience  a sudden dip at $k_f$. The algorithm given in TABLE \ref{tab:tfgard} identify this dip and deliver high quality estimate of $\boldsymbol{\beta}$ without having any tuning parameter. 
\begin{table}\centering
\begin{tabular}{|l|}
\hline 
{\bf Input:-}  Observed vector ${\bf y}$, Design matrix {\bf X}  \\
{\bf Step 1:-} Run GARD for $k_{max}$ iterations. \\
{\bf Step 2:-} Identify $k_f$ as $\hat{k_f}=\underset{1\leq k\leq k_{max}}{\arg\min}{t(k)}$.\\
{\bf Step 3:-} Jointly estimate $\boldsymbol{\beta}$ and ${\bf g}_{j_1},\dotsc,{\bf g}_{j_{\hat{k}_f}}$ as \\
\ \ \ \ \ \ \ \ \ \ $[{\hat{\boldsymbol{\beta}} }^T  [\hat{\bf g}_{j_1},\dotsc,\hat{\bf g}_{j_{\hat{k}_f}}] ]^T= {{\bf A}^{(\hat{k}_f)}}^{\dagger}{\bf y}.$\\
{\bf Output:-} Signal estimate $\hat{\boldsymbol{\beta}}$.\\
\hline 
\end{tabular}
\caption{Tuning free GARD}
\label{tab:tfgard}
\end{table}
\begin{remark}
 The only parameter to be specified in TF-GARD is the maximum iterations $k_{max}$. The maximum number of iterations possible before the matrices ${\bf A}^{(k)}$ becoming rank deficient is  $n-p$. Note that  the objective of sparse outlier modelling is to model few number of gross errors that cannot be modelled by inlier noise. In other words, the model implicitly assumes that $n_{out}\ll n$. Further, the Cospark based analysis in \cite{decoding_candes} reveals that the maximum number of outliers $n_{out}$ that can be tolerated satisfies $n_{out}\leq \floor{\dfrac{\text{Cospark}({\bf X})}{2}}$ and $\text{Cospark}({\bf X})$ satisfies $\text{Cospark}({\bf X})\leq n-p+1$.  Taking these ideas into consideration, we  fix  $k_{max}$ to be $\floor{\dfrac{n-p+1}{2}}$. In any case, one should stop before the matrix ${\bf A}^{(k)}$ is rank deficient or residual ${\bf r}^{(k)}$ is zero. 
 \end{remark}
  A detailed analysis of TF-GARD is not given in this article. However, following the similarities between OMP and GARD, we conjecture that the TF-GARD recover the support of ${\bf g}$ under the same set of conditions used in \cite{gard} albeit at a higher SNR than GARD itself.  Numerical simulations indicate that the  performance of TF-GARD is highly competitive with the  performance of GARD($n_{out}$) and GARD($\sigma^2$)  over a wide range of SNR and SIR. 
\section{Numerical Simulations}
In this section, we numerically evaluate the performance of techniques proposed in this paper \textit{viz}, TF-OMP and TF-GARD  and provide insights into the  strengths and shortcomings of the same. First we consider the case of TF-OMP. We compare the performance of  TF-OMP  with that of OMP($k_0$), OMP$(\sigma^2)$, LASSO\cite{candes2009near} and SPICE\cite{spice,spice_like}. Among these, LASSO and OMP$(\sigma^2)$ are provided with noise variance $\sigma^2$. OMP($\sigma^2$) stop iterations when $\|{\bf r}^{(k)}\|_2\leq \sigma\sqrt{n+2\sqrt{n\log(n)}}$\cite{cai2011orthogonal}.  LASSO in (\ref{lasso}) uses the value $\lambda=2\sigma\sqrt{2\log(p)}$ proposed in \cite{candes2009near}. To remove the bias in LASSO estimate, we re-estimate the non zero entries in LASSO estimate using LS.  As mentioned before, SPICE is a tuning free algorithm. We implement SPICE using the cyclic algorithm proposed in \cite{spice_like}. The iterations in cyclic algorithm is terminated once the difference in the norm of quantity ${\bf p}^i$ in successive iterations are dropped below $10^{-6}$. As observed in \cite{spicenote}, SPICE results in biased estimate. To de-bias the SPICE estimate, we collect the coefficients in the SPICE estimate that comprises $95\%$ of the energy and re-estimate these entries using LS. This estimate  denoted by SPICE($95\%$) in figures exhibits highly competitive performance. All results except the PE vs SNR plot in Fig.\ref{fig.Had32_MSE} and the symbol error rate (SER) vs SNR plot in Fig.\ref{fig.MIMO_OMP} are presented after $10^3$ iterations. These two plots were produced after performing $10^4$ iterations. Unless explicitly stated, the non zero entries of $\boldsymbol{\beta}$ are fixed at $\pm 1$ and the locations of these non zero entries  are randomly permuted in each iteration.  
\subsection{ Small sample performance of TF-OMP. }
In this section, we evaluate the performance of algorithms when the problem dimensions are small. For this, we consider a matrix of the form ${\bf X}=[{\bf I}_n, {\bf H}_n]$, where ${\bf H}_n$ is the $n \times n$ Hadamard matrix. It is well known that the mutual coherence of this matrix is given by $\mu({\bf X})=\dfrac{1}{\sqrt{n}}$\cite{elad_book}. Hence, ${\bf X}$ satisfies the mutual incoherence property whenever $k_0\leq \dfrac{1}{2}(1+\sqrt{n})$. In our experiments we fix $n=32$ and $k_0=3$. Note that $p=2n$ by construction.  For this particular $(n,p,k_0)$, MIC and ERC are satisfied.   

\begin{figure}[htb]
\includegraphics[width=\columnwidth]{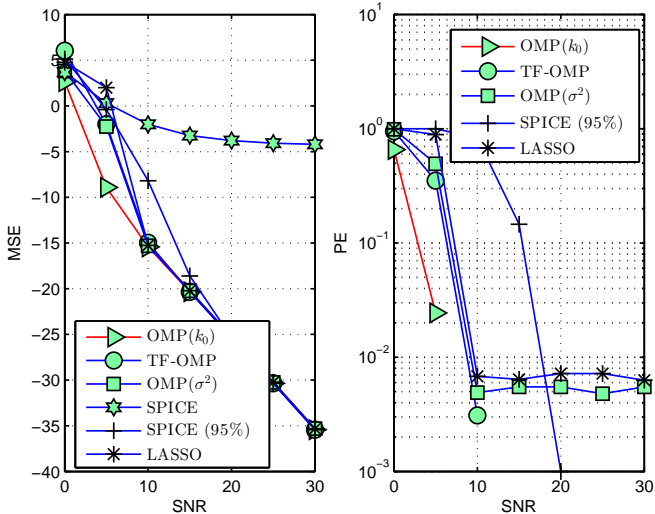}
\caption{ MSE and PE performance when $n=32$ and $k_0=3$ for ${\bf X}=[{\bf I}_n,{\bf H}_n]$. }
\label{fig.Had32_MSE}
\end{figure}

From Fig.\ref{fig.Had32_MSE}, it can be observed that the performance of all algorithms under consideration are equivalent at high SNR in terms of MSE. At low SNR, OMP($k_0$) has the best performance.  The performance of TF-OMP is slightly inferior to OMP($k_0$) at low SNR, whereas it  matches OMP($\sigma^2$) and LASSO across the entire SNR range. TF-OMP is performing better in comparison with both versions of SPICE. This is important considering the fact that SPICE is also a tuning free algorithm. In terms of support recovery error,  OMP($k_0$) has the best performance followed closely by TF-OMP. Both LASSO and OMP($\sigma^2$) are inconsistent at high SNR as proved in \cite{cs_tsp}, whereas TF-OMP is high SNR consistent. This validates Theorems 1-2 in Section \rom{4}. Note that the SPICE estimate contains a number of very small entries which is an artefact of termination criteria. Identifying significant entries from  this estimate in the absence of knowledge of $\boldsymbol{\beta}$ and $\sigma^2$ is difficult and is subjective in nature. We have used a $95\%$ energy criteria to perform this task. However, unlike the MSE performance, we have observed that the $PE$ of SPICE($\alpha\%)$ depends crucially on $\alpha$. 
We choose $95\%$ percent mainly because it gave a very good MSE performance.  As one can observe from Fig.\ref{fig.Had32_MSE}, SPICE$(95\%)$  is also high SNR consistent. However, with a different choice of $\alpha$ one can possibly improve the $PE$ performance. OMP based algorithms being step wise in nature will not have this problem. 
\subsection{ Large sample performance of TF-OMP.  }
In this section, we evaluate the performance of algorithms \\
a).When both $p$ and $k_0$ are fixed and $n$ is increasing  and \\
b).When $p$ is fixed and both $n$ and $k_0$ are increasing.\\
The matrix ${\bf X}$ for this purpose is generated by sampling ${\bf X}(i,j)$ \textit{i.i.d} from a $\mathcal{N}(0,1)$ distribution. Later the columns of ${\bf X}$ are normalised to unit $l_2$ norm.  For the fixed sparsity and increasing $n$ case, all algorithms under consideration except SPICE achieves similar performance. As the number of samples $n$ increase, the MSE improves for all algorithms. In the second case, the sparsity $k_0$ is increased linearly with $n$. From the R.H.S of Fig.\ref{fig.fixedk0_MSE} one can observe that the performance of OMP($k_0$), LASSO and TF-OMP  matches across the $n/p$ ratio under consideration. In particular TF-OMP outperforms both SPICE($95\%$) and OMP($\sigma^2$). 
\begin{figure}[htb]
\includegraphics[width=\columnwidth]{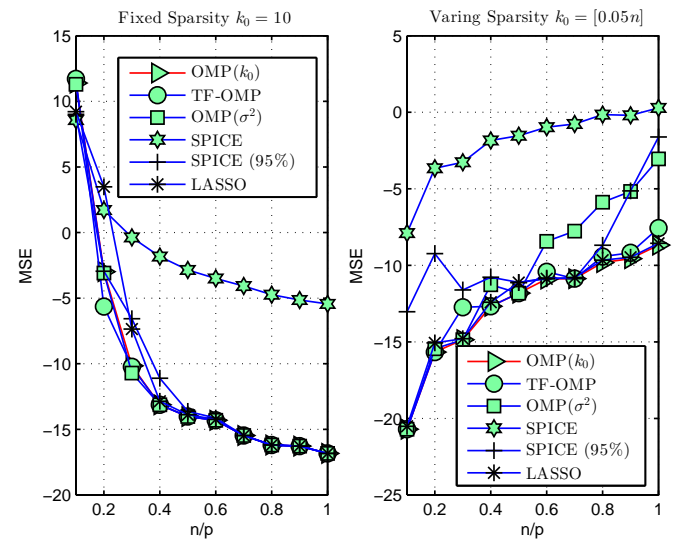}
\caption{ MSE performance when $p=500$, SNR=10 dB and ${\bf X}$ is a Gaussian random matrix. }
\label{fig.fixedk0_MSE}
\end{figure}
\subsection{ Performance of TF-OMP in  signals with high $\boldsymbol{\beta}_{max}/\boldsymbol{\beta}_{min}$.} 
The analysis in Section \rom{4} pointed to a deteriorated performance of TF-OMP when $\boldsymbol{\beta}_{max}/\boldsymbol{\beta}_{min}$ is large. In this section we evaluate this performance degradation numerically. The matrix under consideration is same the $32\times 64$ matrix used in Section \rom{6}.A. The sparsity is fixed at $k_0=3$. However, the magnitude of non zero entries of $\boldsymbol{\beta}$ are $[a,a\alpha,a\alpha^2]$ and the signs are random as before. As the value of $\alpha$ decreases the variation $\boldsymbol{\beta}_{max}/\boldsymbol{\beta}_{min}$ increases. $a$ and $\alpha$ are fixed such that $\|\boldsymbol{\beta}\|_2$ is same as the case when the non zero entries were $\pm1$, i.e., no decay. It is clear from Fig.\ref{fig.exponential_decay} that the performance of TF-OMP indeed deteriorate when non zero entries are decaying and the degradation becomes more and more severe as the decaying factor $\alpha$ decreases. However, as SNR increases the performance of TF-OMP still matches  the performance of OMP$(k_0)$. This again validate Theorem 2, this time for exponentially decaying signals. The other tuning free algorithm under consideration, i.e., SPICE(95\%) also performs poorly in the presence of high $\boldsymbol{\beta}_{max}/\boldsymbol{\beta}_{min}$. Further, unlike TF-OMP, the performance of SPICE(95\%) is not improving with increasing SNR. 
\begin{figure}[htb]
\includegraphics[width=\columnwidth]{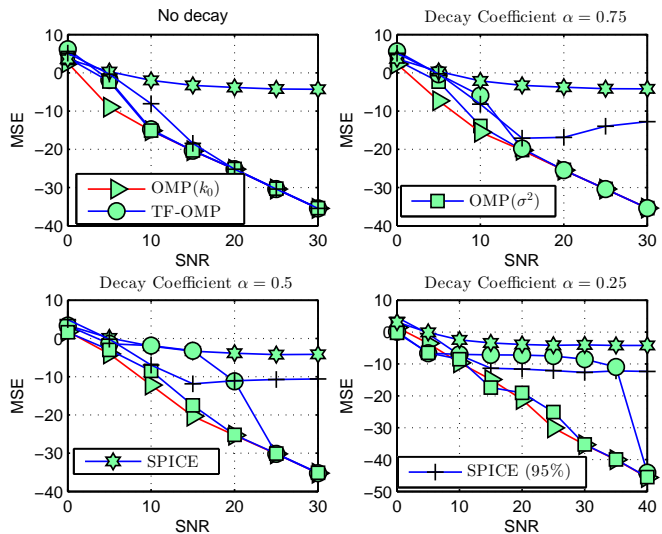}
\caption{ MSE performance when ${\bf X}=[{\bf I}_n,{\bf H}_n]$, $n=32$, $k_0=3$ and non zero entries in $\boldsymbol{\beta}$ are exponentially decaying. }
\label{fig.exponential_decay}
\end{figure}
\subsection{Performance of TF-OMP in the presence of correlated random design matrices}
The performance analysis of TF-OMP was conducted under the assumption of MIC and ERC. In this section, we study the performance of TF-OMP in coherent design matrices where these assumptions are no longer valid. For this purpose we generated  random matrices with $n=32$ and $p=64$ such that all columns ${\bf X}_i$ have unit $l_2$ norm and correlation between ${\bf X}_i$ and ${\bf X}_j$ equals $\rho^{|i-j|}$. As the correlation  factor $0\leq\rho\leq  1$ increases, the correlation between the columns in ${\bf X}$ also increases. $\boldsymbol{\beta}$ has $k_0=3$ non zero entries which are $\pm1$. It can be seen from Fig.\ref{fig.high_correlation} that the performance of all algorithms under consideration degrade with increasing correlation. However, the performance of TF-OMP is much better than OMP($k_0$). This can be attributed to the fact that in highly coherent dictionaries OMP is less likely to cover $\mathcal{I}$ in exactly $k_0$ iterations. At the same time TF-OMP can estimate a superset of entries in $\boldsymbol{\beta}$ containing $\mathcal{I}$ as long as OMP can cover $\mathcal{I}$ within $n/2$ iterations. LASSO and SPICE($95\%$) have the best performance.   The performance of TF-OMP and SPICE($95\%$) are similar in the moderate SNR regime except when correlation factor is very high $\rho=0.75$. However, all algorithms perform poorly at this level of correlation. In fact the MSE of SPICE($95\%$) at $\rho=0.75$ is approximately 12 dB worse than the MSE at $\rho=0$. To summarize, this experiment demonstrate the ability of TF-OMP to perform better than OMP($k_0$) when the design matrix is  coherent. 
\begin{figure}[htb]
\includegraphics[width=\columnwidth]{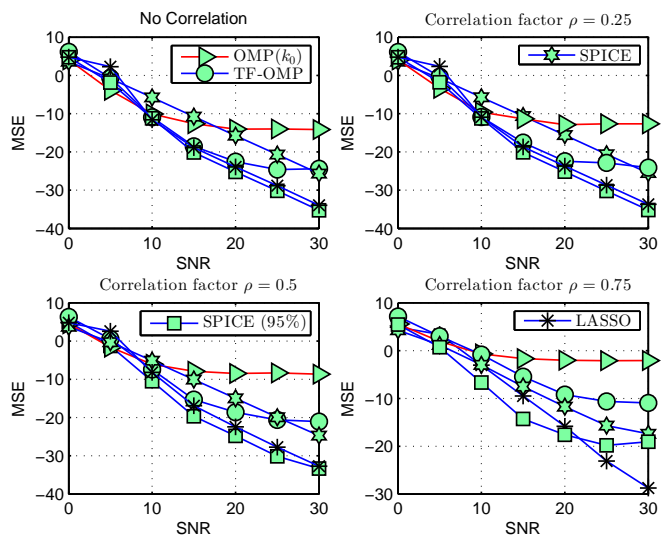}
\caption{ MSE performance when ${\bf X}$ is highly coherent. }
\label{fig.high_correlation}
\end{figure}

\subsection{Performance of low complexity versions of TF-OMP}
Fig.\ref{fig.QTF_OMP} compare the performance of low complexity versions of QTF-OMP1 and QTF-OMP2 with that of TF-OMP and OMP($k_0$). ERC matrix in Fig.\ref{fig.QTF_OMP} denotes simulation setting  considered in Section \rom{6}.A and random matrix denotes the simulation setting in Section \rom{6}.B. It can be seen that the performance of QTF-OMP1 and QTF-OMP2 closely matches the performance of TF-OMP and OMP($k_0$) in the  ERC matrix, in random matrix with $k_0=10$ and $k_0=\floor{0.05n}$. When $k_0=\floor{0.1n}$, it can be seen that the performance of QTF-OMP1 is poorer than the performance of other algorithms when $n/p$ is low. This is because of the fact that $k_{max}$ of QTF-OMP1 at low $n/p$ is lower than $k_0=\floor{0.1n}$ which results in poorer performance. However, $k_0$ is too high in that experiment   and all algorithms under consideration performs poorly. We tabulate the number of iterations of concerned algorithms when $p=500$ and $n$ increases from $100$ to $450$ in TABLE \ref{Tab:QTF}. It can be seen that QTF-OMP1 and QTF-OMP2 requires significantly lower number of iterations in comparison with TF-OMP. To summarize,  it is possible to achieve a performance similar to that of TF-OMP with significantly lower computational complexity using QTF-OMP1 and QTF-OMP2. 
\begin{figure}[htb]
\includegraphics[width=\columnwidth]{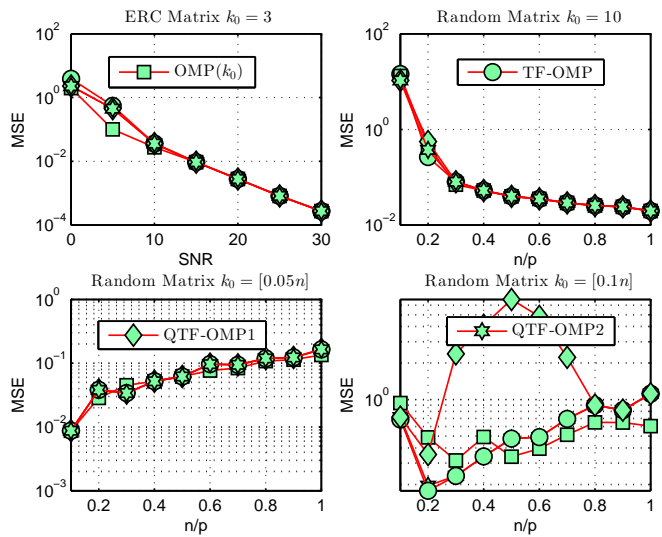}
\caption{ MSE performance of QTF-OMP1 and QTF-OMP2. }
\label{fig.QTF_OMP}
\end{figure}
\begin{table}

\begin{tabular}{|l|l|l|l|l|l|l|l|l| }
\hline
n&100&150&200&250&300&350&400&450 \\
\hline
TF-OMP&50&75&100&125&150&175&200&225\\
\hline
QTF-OMP1&12&15&19&23&28&35&45&68\\
\hline
QTF-OMP2&16&24&32&40&48&56&64&72\\
\hline
 \end{tabular}
 \caption{Number of OMP iterations in TF-OMP, QTF-OMP1 and QTF-OMP2}
 \label{Tab:QTF}
\end{table}
\subsection{Compressive Sensing Based MIMO Detection}
Next we consider a practical application of TF-OMP in the CS based  multiple input multiple output (MIMO) detection framework  proposed in \cite{choi2015sparse} and \cite{choi2015sparse2}. Consider a  MIMO model ${\bf y}={\bf H}{\bf x}+{\bf w}$  with $N_r$ receiver antennas and $N_t$ transmitter antennas. The channel matrix ${\bf H} \sim \mathbb{C}^{N_r \times N_t}$ is assumed to have ${\textit{i.i.d}}\ \mathcal{C}\mathcal{N}(0,1)$ entries and is known completely at receiver. The transmitted vector ${\bf x}$ is modulated using QPSK symbols (i.e., ${\bf x}_j=\pm1\pm i$) and the noise vector  ${\bf w}\sim \mathcal{C}\mathcal{N}({\bf 0}_n,\sigma^2{\bf I}_n)$. Since ML decoding of large scale MIMO systems are NP-hard, low complexity sub optimal detectors are widely preferred. Let $\hat{\bf x}$ be an estimate of ${\bf x}$ using a low complexity MIMO detector. Then it is argued in \cite{choi2015sparse} and \cite{choi2015sparse2} that the error vector ${\bf e}={\bf x}-\hat{\bf x}$ is sparse in the moderate to high SNR regime, i.e., $\|{\bf e}\|_0\ll N_t$. Hence one can estimate ${\bf e}$ from the regression model $\tilde{\bf y}={\bf y}-{\bf H}\hat{\bf x}={\bf H}{\bf e}+{\bf w}$ using CS algorithms and this post processing can be used to correct the error in $\hat{\bf x}$ and improve the SER=$\mathbb{P}({\bf x}\neq \hat{\bf x})$ performance. This framework is generic in the sense that any CS algorithm can be used for CS stage and any algorithm can be used to produce the preliminary estimate $\hat{\bf x}$. Even though this technique is applicable to both overdetermined $N_r\geq N_t$ and underdetermined $N_r<N_t$ MIMO systems, \cite{choi2015sparse} and \cite{choi2015sparse2} considered only the $N_r\geq N_t$ case and  assumed that the noise variance $\sigma^2$ is known. Further to implement the CS stage, it was assumed that the number of errors $k_0=\|{\bf e}\|_0\leq 0.15N_t$. 

In this section, we apply TF-OMP for  detection in both underdetermined and overdetermined MIMO systems. In both cases $\sigma^2$ is assumed to be unknown. ``Algorithm 1+Algorithm 2" in Fig.\ref{fig.MIMO_OMP} represents the performance of CS based MIMO detection with Algorithm 1 in first stage and Algorithm 2 in second stage. In the overdetermined MIMO case shown in the L.H.S of Fig.\ref{fig.MIMO_OMP}, we use the LMMSE estimator in the first stage\cite{choi2015sparse,choi2015sparse2}.  Here, one can estimate $\sigma^2$ using ML estimate of $\sigma^2$ (denoted by $\tilde{\sigma^2}$) and this estimate is used in implementing both LMMSE and OMP($\tilde{\sigma^2}$) algorithms.  From the L.H.S of Fig.\ref{fig.MIMO_OMP} it is clear that the performance of LMMSE+TF-OMP closely matches  the performance of LMMSE+OMP($\tilde{\sigma^2}$) and LMMSE+OMP($k_0$).
\begin{figure}[htb]
\includegraphics[width=\columnwidth]{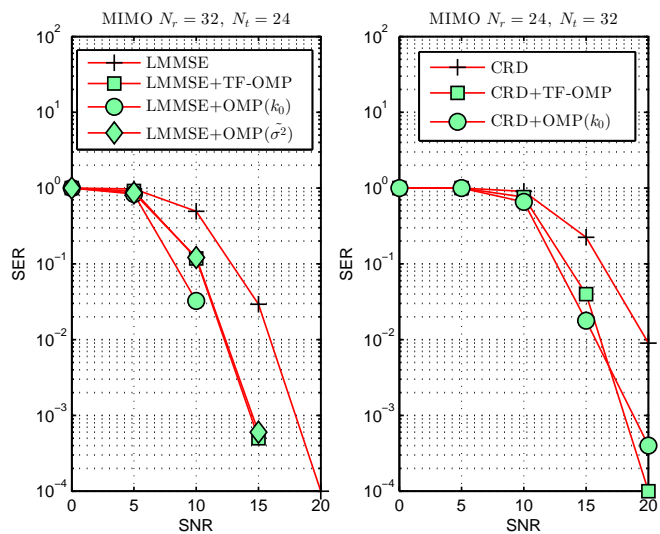}
\caption{ SER performance of TF-OMP based MIMO detection. }
\label{fig.MIMO_OMP}
\end{figure}

A more interesting case is detection in underdetermined MIMO systems where $\sigma^2$ is  non estimable. Hence, both Algorithm 1 and Algorithm 2 must not depend on $\sigma^2$. This means that the LMMSE estimator considered in \cite{choi2015sparse} and \cite{choi2015sparse2} cannot be applied in underdetermined MIMO models in the absence of prior knowledge of $\sigma^2$. We use the widely popular convex relaxation  detector (CRD) which estimate the entries of ${\bf x}_r=[real({\bf x})^T,imag({\bf x})^T]^T$ in the real valued equivalent MIMO model ${\bf y}_r={\bf H}_r{\bf x}_r+{\bf w}_r$ using \cite{MIMO_Lagrangian}
\begin{equation}
\hat{\bf x}_r=\underset{{\bf z} \in \mathbb{R}^{2N_t},-1\leq z_i\leq  1}{\arg\min}\|{\bf y}_r-{\bf H}_r{\bf z}\|_2^2.
\end{equation} 
We implemented this optimization problem using the ``lsqlin" function in MATLAB. Later we produce the preliminary estimate $\hat{\bf x}$ by quantizing $\hat{\bf x}_r(1:N_t)+i\hat{\bf x}_r(N_t+1:2N_t)$. From the R.H.S of Fig.\ref{fig.MIMO_OMP} it is clear that CDR+TF-OMP significantly outperforms CDR as SNR improves. Further, the performance of CDR+TF-OMP is very close to that of CDR+OMP($k_0$). This good performance is achieved without assuming anything about $k_0$ and without knowing $\sigma^2$.  This demonstrate that TF-OMP can be considered as an algorithm of choice for implementing the CS stage in the framework proposed in \cite{choi2015sparse} and \cite{choi2015sparse2} for underdetermined MIMO detection systems when $\sigma^2$ is unknown.

\subsection{Performance of TF-GARD}
Recall  the linear regression model with sparse outlier given by ${\bf y}={\bf X}\boldsymbol{\beta}+{\bf w}+{\bf g}$, where ${\bf w}\sim\mathcal{N}({\bf 0}_n,\sigma^2{\bf I}_n)$ is the inlier and ${\bf g}$ is the sparse outlier. The matrix ${\bf X} \in \mathbb{R}^{n \times p}$ is generated by \textit{i.i.d} sampling from $\mathcal{N}(0,1)$ and later normalised to have unit $l_2$ norm. All entries of $\boldsymbol{\beta}$ are non zero and generated \textit{i.i.d} according to $\mathcal{N}(0,1)$. We fix $n=250$ and $p=30$. The non zero entries of ${\bf g}$ have magnitude $\sqrt{\dfrac{\|{\bf X}\boldsymbol{\beta}\|_2^2}{n_{out}SIR}}$ with random signs. In figures Fig.\ref{fig.robust_SNR} and Fig.\ref{fig.robust_nout}, ``WO" represent the performance of LS estimate of $\boldsymbol{\beta}$ in the absence of outliers and ``LS" represents the performance of LS estimate in the presence of outliers. In effect ``WO" is the best performance one can hope for. GARD($\sigma^2$) is the version of GARD which stops iterations when the residual drop below $\epsilon_{gard}$, a bound on $\|{\bf w}\|_2$. 
This is the tuning parameter for GARD($\sigma^2$). We fix the value of $\epsilon_{gard}$ at $\epsilon_{gard}=\sigma\sqrt{n+2\sqrt{n\log(n)}}$. Note that for ${\bf w}\sim\mathcal{N}({\bf 0}_n,\sigma^2{\bf I}_n)$, $\mathbb{P}\left(\|{\bf w}\|_2>\sigma\sqrt{n+2\sqrt{n\log(n)}}\right)\leq \dfrac{1}{n}$\cite{cai2011orthogonal} and hence this stopping rule   is highly accurate.  M-est use Tukey's bi-weight estimator and is implemented using the MATLAB function ``robustfit" with default settings\cite{maronna2006wiley}. This is also a tuning free robust algorithm.   
\begin{figure}[htb]
\includegraphics[width=\columnwidth]{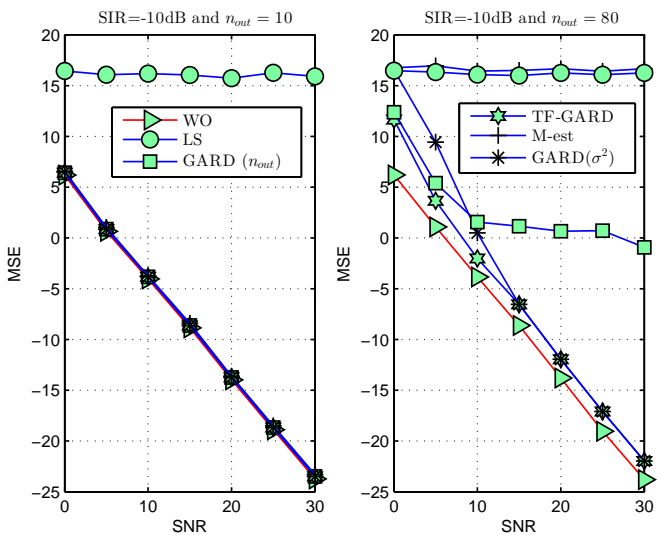}
\caption{ MSE performance for fixed SIR, $n_{out}$ and varying SNR. }
\label{fig.robust_SNR}
\end{figure}

In Fig.\ref{fig.robust_SNR} we compare the MSE performance of algorithms when the SIR is low and SNR is varying. It can be seen that when the number of outliers is low ($n_{out}=10$), the performance of all algorithms matches the performance of LS estimate in the absence of outliers. In other words, all algorithms under consideration are able to mitigate the effect of outliers. However, when the number of outliers is high, i.e., $n_{out}=80$, the performance of all algorithms deviate from the ideal represented by ``WO". However, the deviation from the optimal is minimum  for both TF-GARD and GARD($\sigma^2$) at high SNR. Throughout the low to moderate SNR regime, TF-GARD outperforms GARD($\sigma^2$). Note that TF-GARD is completely oblivious to the inlier noise statistics which is provided to GARD$(\sigma^2)$.  The performance of M-est at this level of $n_{out}$ is very poor and is comparable to that of the ordinary LS estimate.   
\begin{figure}[htb]
\includegraphics[width=\columnwidth]{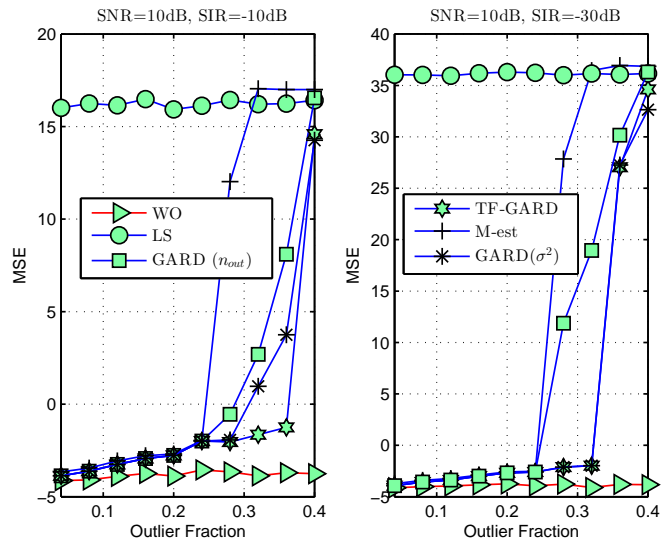}
\caption{ MSE performance for fixed SIR, SNR and varying $n_{out}$. }
\label{fig.robust_nout}
\end{figure}

In Fig.\ref{fig.robust_nout} we present the MSE performance of algorithms when SNR is fixed and the number of outliers are varying. When the number of outliers are increasing, the performance of all algorithms deteriorates. However, the breakdown point of GARD based schemes are much higher than that of M-est. Further, the performance of TF-GARD matches the performance of GARD$(\sigma^2)$ across the entire range of outlier fractions and slightly outperforms the latter in some cases.  Over a wide range of simulations conducted, we have observed that TF-GARD matches the performance of GARD$(\sigma^2)$. Further, the observations made in \cite{gard} about  the relative merits and de-merits of GARD w.r.t algorithms like \cite{error_correction_candes,bdrao_robust} hold true for TF-GARD also.
\section{Conclusion and Future Research}
This article developed a novel OMP based algorithm called TF-OMP that does not require sparsity level $k_0$ or noise variance $\sigma^2$ for efficient operation. TF-OMP is both analytically and numerically shown to  achieve highly competitive  performance in comparison with existing versions of OMP. The operating principle behind TF-OMP is extended to produce TF-GARD which is also exhibiting competitive performance. The broader area of CS involves  many problem scenarios other than the linear regression model considered in this article.  However, most CS algorithms involve tuning parameters that depends on nuisance parameters like noise variance which are difficult to estimate. Hence, it is of tremendous importance to develop  tuning free  and computationally efficient  algorithms like TF-OMP and TF-GARD for other CS applications also.
\bibliography{compressive.bib}
\bibliographystyle{IEEEtran}

\end{document}

%% file: notation.tex
\newcommand{\upperRomannumeral}[1]{\uppercase\expandafter{\romannumeral#1}}

\newtheorem{lemma}{Lemma}{}
  \newtheorem{thm}{Theorem}